\documentclass{elsarticle}

\usepackage{lineno,hyperref}
\modulolinenumbers[5]

\usepackage{graphicx}
\usepackage{amsthm}
\usepackage{amsmath,amsfonts,mathrsfs,amssymb}
\usepackage{bm}
\usepackage{algorithm}
\usepackage[noend]{algorithmic}
\usepackage{hyperref}
\usepackage{subfig}
\usepackage{color}
\usepackage{rotating}
\usepackage{multirow}
\usepackage{placeins}
\usepackage{soul}
\usepackage{xifthen}
\usepackage{url}
\usepackage{longtable}
\usepackage{enumitem}
\usepackage{todonotes}

\renewcommand\algorithmiccomment[1]{%
  \hfill\# {#1}%
}

\newcommand{\ranking}[1][]{%
\ifthenelse{\isempty{#1}}{\bm{r}}{\bm{#1}}%
}
\newcommand{\topset}[2]{S_{\ranking[#1]}^{#2}}

\newcommand{\rank}[1]{\mathit{rank}_{\ranking[#1]}}
\newcommand{\score}{r}

\newcommand{\fuji}{\textsc{Fuji}}
\newcommand{\jaccard}{\textsc{Jaccard}}

\newcommand{\auc}{\operatorname{AUC}}

\newtheorem{proposition}{Proposition}
\newtheorem{theorem}{Theorem}

\makeatletter
\newcommand{\vast}{\bBigg@{4}}
\newcommand{\Vast}{\bBigg@{5}}
\makeatother

\newcommand{\figureBlock}[1]{%
	\begin{figure*}[!h]
		\begingroup
            \captionsetup[subfigure]{width=0.49\textwidth}
            \subfloat[Genie3 and RF rankings on \texttt{#1} data set]{
                \includegraphics[width=0.49\textwidth]{supplementaryGraphsASOC/#1Genie3-RForest}}
        \endgroup
        \begingroup
            \centering
            \captionsetup[subfigure]{width=0.49\textwidth}
            \subfloat[Relief and MI rankings on \texttt{#1} data set]{
                \includegraphics[width=0.49\textwidth]{supplementaryGraphsASOC/#1mutual_info-relief}}
        \endgroup
    \caption{Similarity of the rankings for the data set \texttt{#1}. We make comparisons of (a) the ensemble-based rankings, and (b) Relief and MI rankings.}
    \end{figure*}
}

\journal{Applied Soft Computing}

\begin{document}

\begin{frontmatter}

\title{Fuzzy Jaccard Index: \\A robust comparison of ordered lists}


\author[ijs,mps]{Matej Petkovi\'{c}\corref{corresponding}}
\ead{matej.petkovic@ijs.si}
\author[ijs,mps]{Bla\v{z} \v{S}krlj}
\ead{blaz.skrlj@ijs.si}
\author[ijs,mps]{Dragi Kocev}
\ead{dragi.kocev@ijs.si}
\author[ijs,cam]{Nikola Simidjievski}
\ead{ns779@cam.ac.uk}


\cortext[corresponding]{Corresponding author: matej.petkovic@ijs.si, +38614773635}

\address[ijs]{Department of Knowledge Technologies, Jo\v{z}ef Stefan Institute, Jamova 39, Ljubljana, Slovenia}
\address[mps]{International Postgraduate School Jo\v{z}ef Stefan, Jamova 39, Ljubljana, Slovenia}
\address[cam]{Department of Computer Science and Technology, University of Cambridge, 15 JJ Thomson Ave, Cambridge, United Kingdom}




\begin{abstract}
We propose Fuzzy Jaccard Index (\fuji{}) -- a scale-invariant score for similarity assessment of two ranked/ordered lists. \fuji{} improves upon the Jaccard index by incorporating a membership function that takes into account the particular ranks, thus producing both more stable and more accurate similarity estimates. We provide theoretical insights into the properties of the \fuji{} score as well as propose an efficient algorithm for computing it. We also present empirical evidence of its performance in different synthetic scenarios. Finally, we demonstrate its utility in a typical machine learning setting -- comparing feature ranking lists, relevant to a given machine learning task. In many practical applications, in particular originating from high-dimensional domains, where only a small percentage of the whole feature space might be relevant, a robust and confident feature ranking leads to interpretable findings, efficient computation and good predictive performance. In such cases, \fuji{} correctly distinguishes between existing feature ranking approaches, while being more robust and efficient than the benchmark similarity scores.
\end{abstract}


\end{frontmatter}


\section{Introduction}
\label{sec:intro}

A set similarity score quantifies the discrepancy between two nonempty sets. The utility of these similarity scores has been demonstrated in various applications related to information retrieval \cite{Niwattanakul2013,Li2012,Zhang2012}, recommender systems \cite{GAN2013811,info-fusion:ranking-products}, gene selection \cite{OrderedListBioConductor} etc. In the context of practical applications in machine learning and data mining, such scores pertain to different aspects of data pre-processing \cite{Nogueira}, method design \cite{info-fusion:movie-ranking,info-fusion:web-results} and evaluation \cite{pog,nPOG}. While there is a plethora of different similarity scores used for various tasks and settings \cite{EncyclopediaofDistances}, here we focus on set-similarity scores that operate on ordered sets i.e. lists. Their function is systematically described several properties~\citep{Nogueira} that determine whether a particular score:

\begin{enumerate}[label=(\roman*)]
\item is \textbf{fully defined} for all non-empty lists, 
\item is \textbf{bounded}, i.e., does the score take values only from a finite interval $[a, b]\subset\mathbb{R}$,
\item achieves \textbf{maximum} \emph{iff} the two lists are equal, 
\item includes \textbf{correction for chance}, i.e., what is the expected score value for two randomly ordered lists, and 
\item is a \textbf{monotone} function on the intersection of the lists.
\end{enumerate}

A typical approach for determining set-similarity is the Jaccard Index \cite{Jaccard1901}, one of the most widely used similarity measures for sets employed across a variety of domains. In a machine learning setting it is typically used to measure the stability of feature rankings \cite{Khoshgoftaar2013,Saeys2008}, or as part of a heuristic for feature selection \cite{Zou2016}, etc. However, despite its popularity, the standard Jaccard Index is often unstable and does not have a mechanism for correction of chance. Other, more sophisticated, measures that attempt to address the limitations of Jaccard Index include: POG \cite{pog} and its normalized counterpart nPOG \cite{nPOG}, Hamming \cite{hamming}, Kuncheva \cite{kuncheva}, Lustgarten \cite{Lustgarten}, Wald \cite{Wald}, Pearson \cite{Nogueira}, Kr\'{i}zek \cite{Krizek}, CW$_\text{rel}$ \cite{cwrel}, Fuzzy (Goodman and Kruskal's) gamma coefficient \cite{Boucheham2014,Henzen2015} as well as ordinary correlation based on the ranking scores. However, all of these methods, with the exception of correlation, completely ignore the ranking scores and rely solely on the ordering. This often makes them often unstable and unable to accurately detect apparent similarities.

In this paper, we address these shortcomings and propose an improved score -- the Fuzzy Jaccard Index (\fuji{}). \fuji{} builds upon the Jaccard index by incorporating a membership function that takes into account the given ranking scores, leading to more stable and more accurate similarity estimates. We provide theoretical properties of \fuji{} and highlight its benefits on three illustrative synthetic scenarios, comparing it to other similarity scores used in practice. Moreover, we demonstrate the utility of \fuji{} in a typical machine learning setting -- comparing feature-ranking lists relevant to a given machine learning task. 

Feature ranking~\cite{Saeys2007} is a machine learning task where the goal is to obtain a list of features ordered by their relevance to a particular task \cite{Nilsson:JMLR:2007}. It is closely related to the task of feature selection \citep{Saeys2007,Stanczyk2015} that aims to find the smallest subset of relevant $k$ features, which yield accurate predictions for a particular dataset. Both tasks are a direct response to the current trend of ever-increasing amounts of data, with datasets being extremely high-dimensional, thus challenging the typical machine learning process. In practice, this is often referred to as "the curse of dimensionality". It is a major issue for many machine learning algorithms in general, especially when applied to tasks such as biomarker discovery \citep{He2010,Statnikov2005,Xia2013} -- where even though there are many available features only a small subset is relevant \cite{abeel2009robust,skrlj2021reliefe}. Moreover, in addition to performing feature selection,  feature ranking algorithms also provide explanations of the model and its decisions \cite{info-fusion:xai}, particularly useful for post-hoc interpretability analysis in various supervised~\cite{Breiman2001,Chen2016} and unsupervised~\cite{Petkovic2021_UFR} settings.

Even though there exist many (feature) ranking algorithms, it remains an important research question to understand how they relate to each. Moreover, a key aspect of evaluating a ranking algorithm is also estimating its stability. Both tasks amount to measuring similarities, either between rankings produced by different algorithms, or rankings produced by the same algorithm given some data variations (e.g., when performing cross-validation \cite{Boucheham2014}). In this context, we show that \fuji{} is able to provide stable and confident estimates of the similarities between different ranking algorithms. In summary, the contributions of this work are multi-fold:

\begin{enumerate}
    \item We propose Fuzzy Jaccard Index (\fuji{}) -- a scale-invariant score for similarity assessment of ordered lists. We provide theoretical insight into the properties of the \fuji{} score, and find its exact lower bounds, as well as the lower bounds for the area under the \jaccard{} and \fuji{} curves. Moreover, we find (existing) score-minimizing pairs or prove that they do not exist. Finally, we show that \fuji{} is a generalization of \jaccard{}. 
    \item We perform a systematic synthetic study that highlights the limitations of other (standard) similarity measures and demonstrate how they are resolved with \fuji{}.
    \item We present an efficient algorithm (and provide its implementation) for computing \fuji{} similarity scores. 
    \item We empirically show that \fuji{} should be preferred over \jaccard{}, or any of the other 10 similarity scores, in a wide range of real-world scenarios by comparing different feature ranking methods for various classification tasks. This includes a case study on the \texttt{genes} dataset, a collection of gene expression measures from the TCGA pan-cancer study \cite{weinstein2013cancer}.
\end{enumerate}

The code for reproducing the results from this work is available \href{https://github.com/Petkomat/fuji-score}{\color{blue}{\textbf{here}}}.

\section{Fuzzy Jaccard Index (\fuji{})}
In this section, we begin by introducing and explaining the \fuji{} score. More specifically, we describe its properties and demonstrate its utility on three synthetic scenarios, where we compare \fuji{} to other benchmark similarity scores, outlined in the previous section.

\subsection{Preliminaries}
We refer to an ordered list accompanied with ranking scores as ranking. Given a set of $n > 1$ items $x_i$, $1\leq i \leq n$, the ranking is defined by a vector $\ranking{} = (\score{}_1, \dots, \score{}_n)$, where $\score{}_i$ is a relevance score of the item $x_i$.
The rank of the item $x_i$ in the ranking $\ranking{}$ is denoted by $\rank{}(x_i)\geq 1$.

Following the notation of order statistic, we denote the $i$-th ranked item by $x_{(i)}$ and its relevance score by $\score{}_{(i)}$, thus  $\rank{}(x_{(i)}) = i$.
We assume that each score has a distinct non-negative value.

Consider two ranking sets $\ranking[r]$ and $\ranking[s]$. The standard Jaccard Index (\jaccard{}) is based on computing the similarity between subsets of the top-ranked items of $\ranking[r]$ and $\ranking[s]$. In particular, for a given a ranking $\ranking[]$ of $n$ items $x_i$, we denote these subsets as $\topset{r}{k}$, namely $\topset{r}{k} = \{x_{(i)} \mid 1 \leq i\leq k\}$, for all $k\leq n$. The \jaccard{} score is defined as
\begin{equation}
    \label{eqn:jaccard}
    \jaccard{}(\ranking[], \ranking[s], k) = \left| \topset{r}{k} \cap \topset{s}{k}\right| / \left| \topset{r}{k} \cup \topset{s}{k} \right|\text{,}
\end{equation}
where $|S|$ denotes the cardinality of the set $S$. If we introduce the membership function $\mu_S$ of a set $S$, i.e., $\mu_S(x) = 1$ if $x\in S$, and $\mu_S(x) = 0$ otherwise. Therefore the set sizes from Eq.~\eqref{eqn:jaccard} can be rewritten as
\begin{eqnarray}
    \left| \topset{r}{k} \cap \topset{s}{k}\right| &=& \sum_{x\in \topset{r}{k}\cup \topset{s}{k}} \min \{\mu_{\topset{r}{k}}(x), \mu_{\topset{s}{k}}(x) \}, \label{eqn:intersection}\\
    \left| \topset{r}{k} \cup \topset{s}{k}\right| &=& \sum_{x\in \topset{r}{k}\cup \topset{s}{k}} \max \{\mu_{\topset{r}{k}}(x), \mu_{\topset{s}{k}}(x) \}. \label{eqn:union}
\end{eqnarray}

The membership function used in the standard \jaccard{}, however, comes with several undesirable properties when comparing rankings. Namely, since \jaccard{} relies only on the ordering of the items and ignores the actual score values, in many scenarios, it is unstable and unable to accurately detect apparent similarities (and dissimilarities), as presented in Sec.~\ref{sec:examples}.

\subsection{Method Definition}

The Fuzzy Jaccard Index (\fuji{}) extends the standard \jaccard{} and overcomes many of its limitations. In particular, given a ranking $\ranking[] = (r_1, \dots, r_n)$, we define the membership function as

\begin{equation}
    \label{eqn:fuzzy-mu}
    \mu_{\topset{}{k}}^{F}(x_i)= 
    \begin{cases}
    \hphantom{ii.}1 &;\; x_i \in \topset{}{k} \\
    \score{}_i / \score{}_{(k)} &;\; x_i \notin \topset{}{k} \land \score{}_{(k)} > 0 \\
       \hphantom{ii.}0 &;\; \text{otherwise}
    \end{cases}\text{.}
\end{equation}

Note that $r_{(k)}$ (relevance of $x_{(k)}$) is the minimal relevance of the items in $\topset{}{k}$. Thus, we effectively extend the sets $\topset{}{k}$ by allowing other items with similar (but lower) relevance to $r_{(k)}$ to be considered when computing the intersection given in Eq.~\eqref{eqn:intersection}. Finally, following Eq.~\eqref{eqn:jaccard}, \eqref{eqn:intersection} and \eqref{eqn:union}, \fuji{} is defined as
\begin{equation}
    \label{eqn:fuji}
    \fuji{}(\ranking[], \ranking[s], k) = \frac{\sum_{x\in \topset{r}{k}\cup \topset{s}{k}} \min \{\mu_{\topset{r}{k}}^{F}(x), \mu_{\topset{s}{k}}^{F}(x) \}}{\sum_{x\in \topset{r}{k}\cup \topset{s}{k}} \max \{\mu_{\topset{r}{k}}^{F}(x), \mu_{\topset{s}{k}}^{F}(x) \}}
\end{equation}
As shown later, this gives more stable and more accurate similarity estimates. Note also that by computing only relative scores $r_i/r_{(k)}$ \fuji{} becomes scale-invariant.

\subsection{Area Under the \fuji{} Curve}
When comparing rankings, \fuji{} can be further used for constructing a curve, where each point $k$ of the curve refers to 

\begin{equation*}
(k, f_k) = (k, \fuji{}(\ranking[], \ranking[s], k)) ;\;  1\leq k\leq n. 
\end{equation*}

In turn, inspecting such a curve can reveal how similar two rankings are up to some arbitrary cut-off point $k_0$ (e.g., Fig.~\ref{fig:reversed}). Moreover, a more condensed approach to the said inspection is to compute the area under the constructed curve ($\auc{}$). This allows for a clearer and more intuitive comparison of different pairs of rankings, especially in cases when comparing multiple intersecting curves. Given a \fuji{} curve, computing $\auc{}_{\fuji{}}$ is straightforward: using trapezoidal rule and scaling the area to the interval $[0, 1]$ by the factor $1/(n - 1)$, i.e.,

\begin{equation}
    \label{eqn:auc}
    \auc{}_{\fuji{}}(\ranking[r], \ranking[s]) = \frac{1}{n - 1}\left(\frac{f_1 + f_n}{2} + \sum_{k = 2}^{n - 1} f_k \right)\text{.}
\end{equation}

An efficient algorithm for computing the \fuji{} curve is presented in the supplementary material (Appendix C). Efficiency is mostly gained by avoiding unnecessary calculations: First, the values $f_k$ are computed incrementally, and second, the algorithm benefits from storing the current intersection and symmetric difference of the sets of top $k$ items instead of their intersection and union.

\subsection{Synthetic Scenarios}\label{sec:examples}

{\bf Reverse rankings.} Consider two rankings $\ranking[r]$ and $\ranking[s]$ of $n = 10$ items where
\begin{eqnarray}
    \ranking[r] =& (100, 100 - \Delta, \dots, 100 - (n - 1)\Delta)\label{eqn:r2}\\
    \ranking[s] =& (100 - (n - 1)\Delta, \dots, 100 - \Delta, 100)\label{eqn:s2}
\end{eqnarray}

For large values of $\Delta$, the similarity of $\ranking[r]$ and $\ranking[s]$ should be small. Analogously, for really small values of $\Delta$ these two rankings should be considered similar. Such behavior often occurs in real-world applications, when items have (approximately) the same relevance but some randomness is included in the ranking algorithm. Fig.~\ref{fig:reversed} reveals that \fuji{} is capable of detecting this: with $\Delta$ approaching $0$, the \fuji{} values approach $1$, while being low for larger values of $\Delta$. 

In comparison to other scores, shown in Fig.~\ref{fig:manyMeasures:reversed}, we can see that none, except \fuji{}, can accurately detect that the similarity between the rankings increases with $\Delta$. This also holds for the Gamma and correlation scores, both of which still have a constant value of $-1$, even though the Gamma score is fuzzy and the correlation score takes the ranking values into account.

\begin{figure}[htb]
    \centering
    \includegraphics[width=0.6\textwidth]{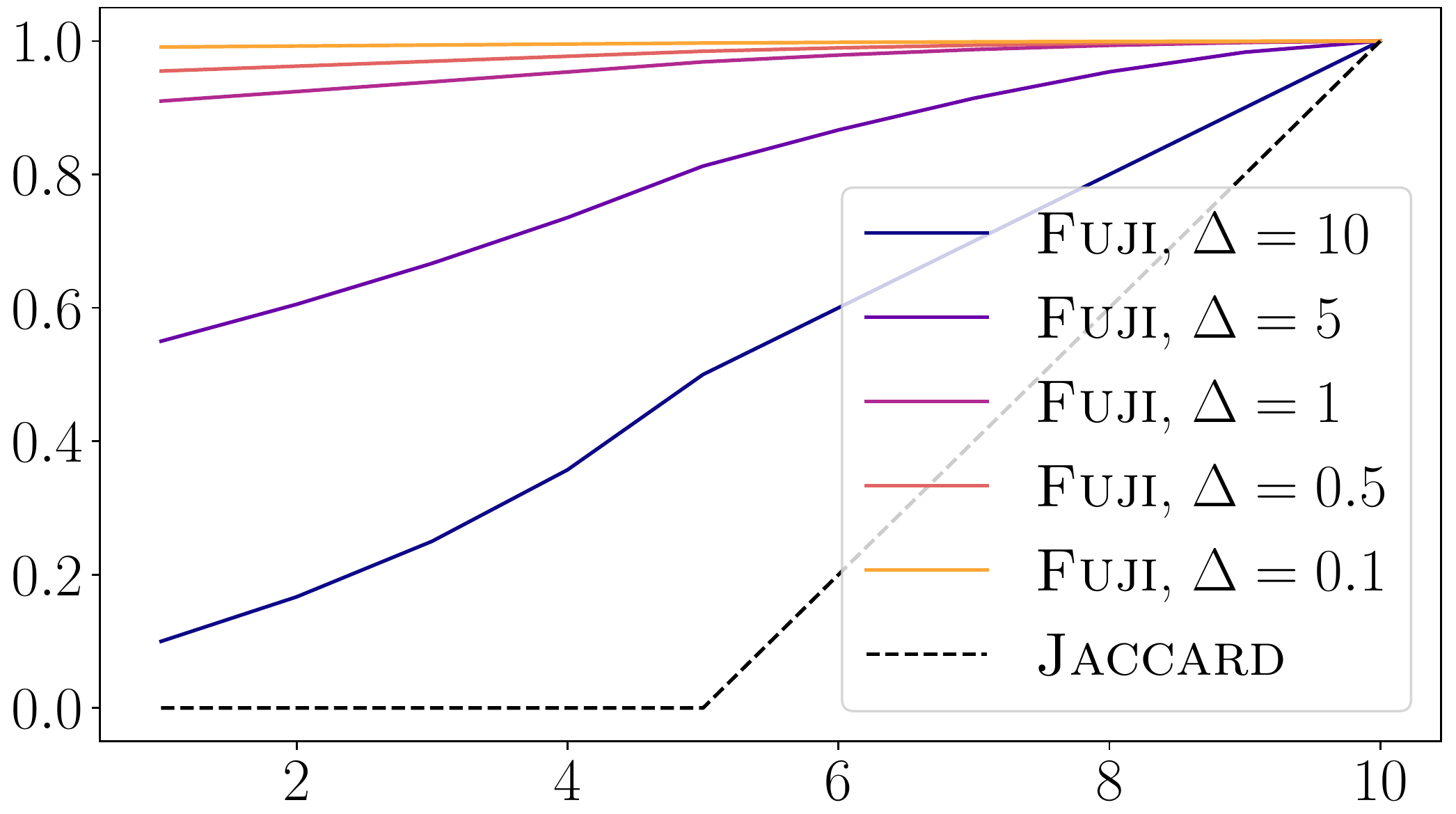}
    \caption{\fuji{} and \jaccard{} values for different numbers of top-ranked items (x-axis), for the reversed rankings $\ranking[r]$ and $\ranking[s]$ (Equations \eqref{eqn:r2} and \eqref{eqn:s2}), for different values of $\Delta$.}
    \label{fig:reversed}
\end{figure}

{\bf Correlated items.} It often happens that the data contain sets of correlated (or at least dependent) items. For instance, \cite{Yan2015} study the stability of recursive feature elimination methods for the gas-sensor data. Items from each ranking, therefore, have similar relevance but may be assigned (slightly) different scores, caused by implicit noise in the data or a random component in the algorithm itself. We simulate this scenario by considering two rankings of $n = 10$ items that come in pairs. The pairs have the same positions in both sets, however, the items in each pair appear in the opposite order:

\begin{eqnarray}
    \ranking[r] = (100, 100 - \Delta, 80, 80 - \Delta, \dots, 20, 20 - \Delta)\hphantom{iii}\label{eqn:r1}\\
    \ranking[s] = (s_1, \dots, s_n),\; \text{where} \; s_{2i} = r_{2i - 1}, s_{2i - 1} = r_{2i}\hphantom{ii} \label{eqn:s1}
\end{eqnarray}

\begin{figure}[htb]
    \centering
    \includegraphics[width=0.6\textwidth]{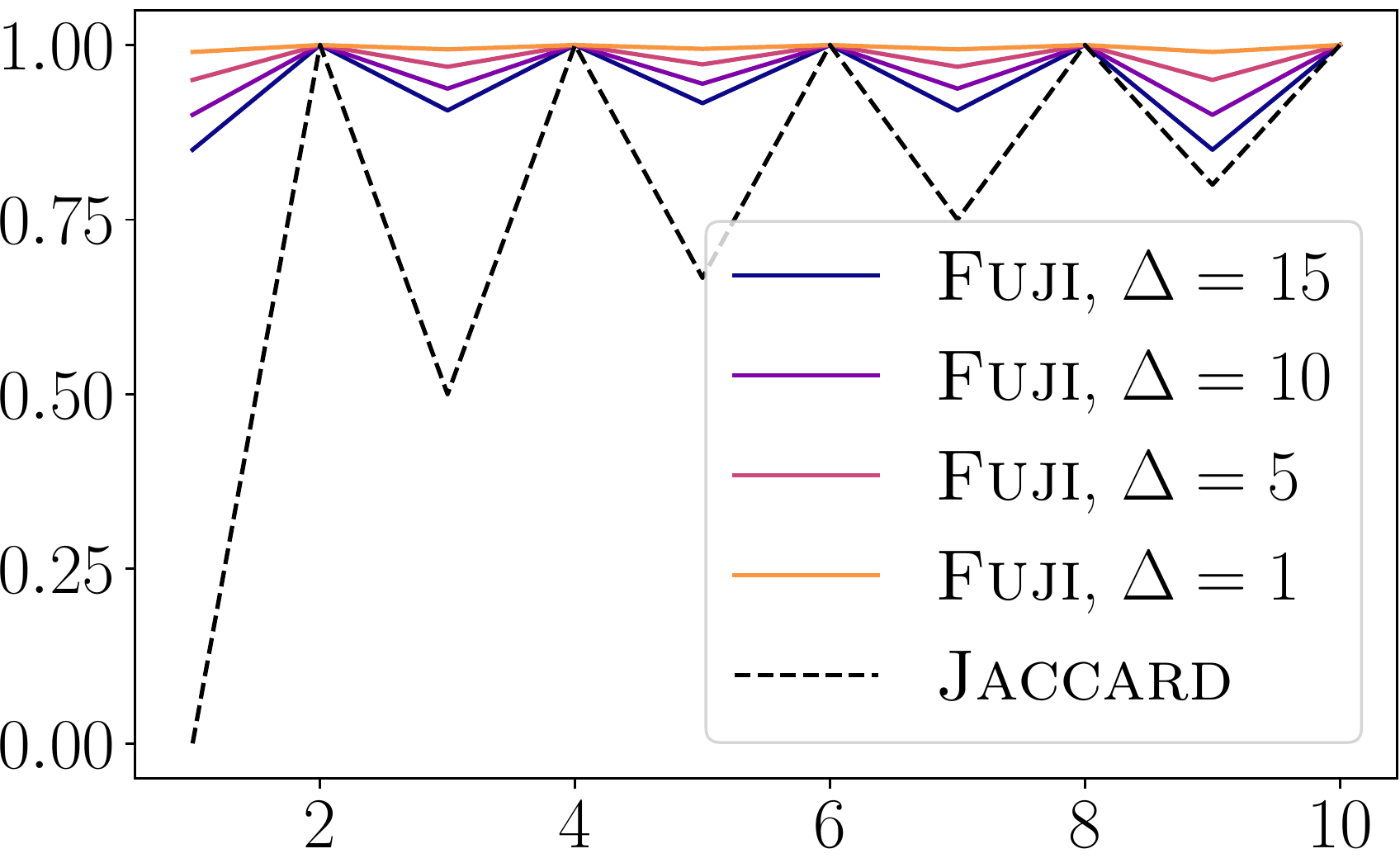}
    \caption{\fuji{} and \jaccard{} values for different numbers of top-ranked items (x-axis), for the rankings $\ranking[r]$ and $\ranking[s]$ (Equations \eqref{eqn:r1} and \eqref{eqn:s1}) of correlated items, for different values of $\Delta$.}
    \label{fig:transpositions}
\end{figure}

\begin{figure*}[htb!]
\centering
\begingroup
    \captionsetup[subfigure]{width=0.47\textwidth}
    \subfloat[\label{fig:manyMeasures:reversed}Similarity of the rankings from Eq.~\eqref{eqn:r2} and Eq.~\eqref{eqn:s2} (reversed).]{
    \includegraphics[width=0.7\textwidth]{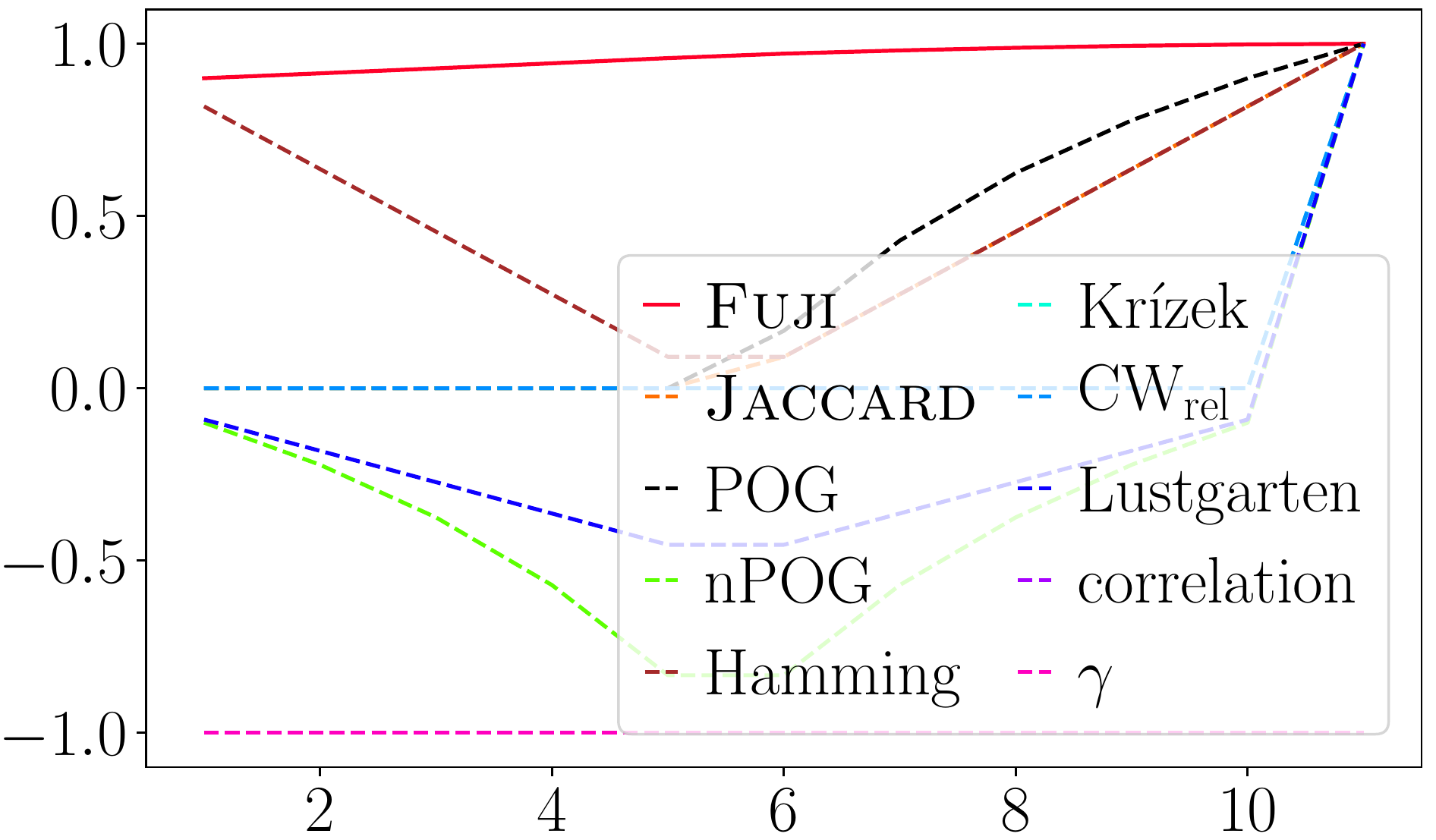}}
\endgroup
\begingroup
    \centering
    \captionsetup[subfigure]{width=0.47\textwidth}
    \subfloat[\label{fig:manyMeasures:transposition}Similarity of the rankings from Eq.~\eqref{eqn:r1} and Eq.~\eqref{eqn:s1} (correlated).]{
    \includegraphics[width=0.7\textwidth]{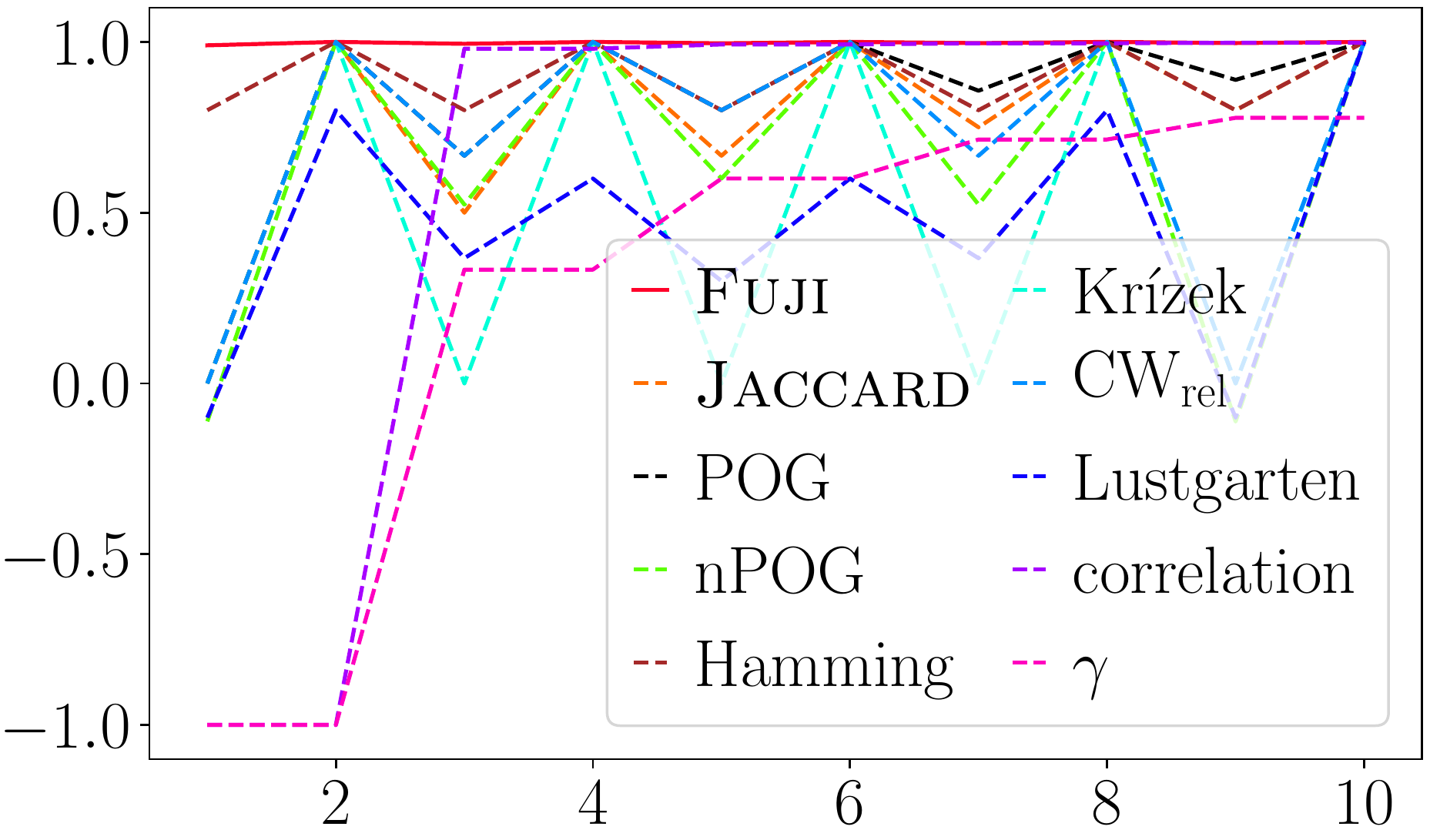}}
\endgroup

\caption{Values of the considered similarity measures for different numbers of top-ranked items (x-axis), for (a) the reversed rankings, and (b) rankings of
correlated items. For readability, we only show the curves for $\Delta = 1$. The legend applies to both graphs. When the rankings have the same length, nPOG, Kuncheva, Wald, and Pearson are guaranteed to return the same curves, therefore, only nPOG curve is shown.}
    \label{fig:manyMeasures}
\end{figure*}


As we can see in Fig.~\ref{fig:transpositions}, the \jaccard{} score is unstable and unreliable. Independently of the $\Delta$ parameter, it increases from $0$ at $k = 1$ to $1$ at $k = 2$, then it decreases to $0.5$ at $k = 3$ etc. The \fuji{} score, on the other hand, is more stable and correctly detects a mostly similar ordering of the items, especially in cases where the $\Delta$ values are low. The comparison of \fuji{} with the other similarity measures is shown in Fig.~\ref{fig:manyMeasures:transposition}. For readability, we only show the results for $\Delta = 1$.

Once again, none of the other benchmark scores can detect these similarities, because they only consider the beginning of the rankings (up to 2 items). For example, the Gamma score is quite unstable: it equals $-1$ at the beginning, gradually increasing to 0.77 by the end. On the other hand, in the case of correlation, this is only the case before the inclusion of the third item. Note that, \fuji{} and correlation are also the most stable methods which further highlights the benefit of taking the ranking scores into account.

{\bf Two-part ranking.} Consider three rankings $\ranking[r]$, $\ranking[s]$ and $\ranking[t]$ that rank the items as follows:
\begin{eqnarray*}
    \ranking[r]: && x_1, x_2, \dots, x_n \label{eqn:ranking-r}\\
    \ranking[s]: && x_{n/2 + 1}, x_{n/2 + 2},  \dots , x_n,\; x_1, x_2, \dots, x_{n/2}  \\
    \ranking[t]: && x_{n}, x_{n - 1},\dots, x_{1}\label{eqn:ranking-t}
\end{eqnarray*}
(we assume that $n$ is an even number).
Note that ranking $\ranking[t]$ has the reverse order of ranking $\ranking[r]$, and that ranking $\ranking[s]$ is obtained from $\ranking[r]$ by dividing it into two parts of length $n/2$ and reversing their order. As a consequence, many relative positions of the items remain the same.

Direct computation shows that $\jaccard{}(\ranking[r],\ranking[s], k) = \jaccard{}(\ranking[r],\ranking[t], k)$, for all $k$,
even though, as presented later in Proposition \ref{pro:jaccard:min}, ranking $\ranking[t]$ is as distant from $\ranking[r]$ as possible.
The same holds for the other similarity scores that do not take ranking scores into account.
In contrast, in the case of \fuji{} one can show that $\auc{}_{\fuji{}}(\ranking[r],\ranking[s]) > \auc{}_{\fuji{}}(\ranking[r],\ranking[t])$. Note that, this also holds for correlation.

\section{Properties}

We consider the following five properties of a similarity score~\cite{Nogueira}:
1) fully defined, 2) boundedness, 3) maximum, 4) correction for chance and 5) utility of ranking scores. The characterization of \fuji{} and the other considered similarity scores, in terms of these properties, is given in Tab.~\ref{tab:properties}.

Although the properties, presented in Section\ref{sec:intro}, are intuitively desirable, the illustrative scenarios in Sec.~\ref{sec:examples} show that taking ranking scores into account is a very beneficial property when comparing ordered lists. Since this may contradict the monotonicity property, we omit monotonicity from the analysis. Note that, all the proofs of the properties of the benchmark scores, except correlation, are presented in \cite{Nogueira}. The proofs for the properties of correlation are given in Appendix B.

\begin{table*}[ht]
    \centering
    \caption{The properties of the similarity scores. ($^*$) Correction for chance in the case of \fuji{}  and \jaccard{}
    is further discussed in Sec.~\ref{sec:correction-for-chance}.}\label{tab:properties}
    \resizebox{\textwidth}{!}{
    \begin{tabular}{l | c c c c c}
    \hline
         Score & Fully defined & Bounded & Maximum & Correction for chance &  Rank scores \\
         \hline
         \hline
        \textbf{\fuji{}} & \checkmark  & \checkmark & \checkmark & \checkmark$^*$ & \checkmark \\
    \hline    \jaccard{} & \checkmark & \checkmark & \checkmark & &  \\
   \hline     Hamming  & \checkmark  & \checkmark &  \checkmark &  & \\ 
   \hline     POG  & \checkmark & \checkmark & \checkmark & & \\ 
    \hline    nPOG & \checkmark & & \checkmark & \checkmark \\ 
    \hline    Kuncheva &  & \checkmark & \checkmark & \checkmark \\ 
   \hline     Lustgarten & \checkmark  & \checkmark & & \checkmark\\ 
    \hline    Wald & \checkmark  & & & \checkmark \\
   \hline     Kr\'{i}zek  &  &  & \checkmark & \\
  \hline      $\text{CW}_\text{rel}$ & \checkmark  & \checkmark & & \\
  \hline      correlation & & \checkmark &  & \checkmark & \checkmark \\
  \hline      Pearson & \checkmark & \checkmark &  \checkmark & \checkmark & \\
    \hline    Gamma ($\gamma$) &  & \checkmark & \checkmark &  &  \\
        \hline
    \end{tabular}
    }
\end{table*}

\subsection{Fully defined}

Since the experiments (and illustrative examples) that we carried out to show the appropriateness of \fuji{} deal with ranking
and comparison of the sets of the top $k$ ranked items, the definition of \fuji{} given in Eq.~\eqref{eqn:fuji} assumes that the sets
$\topset{r}{k}$ and $\topset{s}{k}$ are of equal size $k$. However, the definition can be generalized from $\fuji{}(\ranking[r]{}, \ranking[s]{}, k)$ to $\fuji{}(\ranking[r]{}, \ranking[s]{}, k_1, k_2)$ by using the sets $\topset{r}{k_1}$ and $\topset{s}{k_2}$. Thus, $\fuji{}$ is applicable to any pair of nonempty sets.

\subsection{Boundedness and Maximum}

In this part, we find the exact lower and upper bounds and specify conditions when they are achieved. Since \fuji{} is based on \jaccard{}, we also highlight the important differences and similarities to the standard \jaccard{}. The proofs of the propositions and the theorem are given in the supplementary material, (Appendix A). To simplify the proofs and notation, we assume that we always compare two rankings of the top $k$ ranked items.

{\bf The upper bound of \fuji{} and \jaccard{}.} \fuji{} and \jaccard{} have the same exact upper bound of $1$. For \jaccard{}, this follows immediately from the definitions \eqref{eqn:jaccard} and
and the fact that $|A\cup B| \leq |A\cup B|$ for any two sets. Similarly, this follows for \fuji{} from definitions \eqref{eqn:fuzzy-mu} and \eqref{eqn:fuji}.
As a consequence, this also holds for $\auc{}$ in Eq.~\eqref{eqn:auc}.
The reverse is also true: If the sets $\topset{r}{k}$ and $\topset{s}{k}$ of the items have similarity $1$, then $\topset{r}{k} = \topset{s}{k}$,
hence $\auc{}(\ranking[r], \ranking[s]) = 1$ if and only if $\rank{r}(x_i) =\rank{s}(x_i)$ for all items $x_i$.
Hence, the maximum property is also satisfied.

{\bf The lower bounds of \fuji{} and \jaccard{}.}
Since the membership function $\mu$ is non-negative, a trivial lower bound for  $\fuji{}$ (and \jaccard{}) as well as the derived $\auc{}$ is $0$.
However, since $\fuji{}(\ranking[r], \ranking[s], n) = \jaccard{}(\ranking[r], \ranking[s], n) = 1$, $0$ is not the exact lower bound for all $k$. Next, we derive the exact lower bounds and find minimizing pairs (if they exist) of rankings for which the bound is achieved. 

\begin{proposition}
\label{pro:jaccard:min}
When using \jaccard{}, the least similar ranking to a given reference ranking $\ranking{}$ is every ranking $\ranking[s]$,
such that $\rank{s}(x_i)  = n + 1 - \rank{}(x_i)$. 
\end{proposition}
The property $\rank{s}(x_i)  = n + 1 - \rank{}(x_i)$ is actually the precise definition of the reversed pair of rankings, presented in Sec.~\ref{sec:examples}.
The proof of the proposition also shows that no point of any $\jaccard{}$ curve can lay strictly below the curve that is computed from the two reversed rankings.

\begin{proposition}\label{pro:fuji:min}
Let $S$ be the set of rankings generated by all permutations of the fixed scores $s_{(i)}$, and $\ranking{}$ an arbitrary ranking.
To the ranking $\ranking[]$ least similar element of $S$, as measured by \fuji{}, is the
ranking $\ranking[s]$, such that $\rank{s}(x_i)  = n + 1 - \rank{}(x_i)$. 
\end{proposition}
This is the strongest obtainable result, since  $\inf_{\ranking[s]} \auc{}_{\fuji{}}(\ranking[], \ranking[s])$ is not achieved
if the scores $s_{(i)}$ are not fixed, as evident from (the proof of) the Proposition \ref{pro:fuji:argmin}.
If the scores are not fixed, we can also find the lists $\ranking[s]$ and $\ranking[t]$, such that i) $\rank{r}(x_i) = n + 1 -\rank{s}(x_i)$, for all
$i$, ii) $\rank{r}(x_i) \neq n + 1 -\rank{t}(x_i)$ for some $i$, and iii) rankings $\ranking[r]$ and $\ranking[s]$ are more similar than rankings  $\ranking[r]$ and $\ranking[t]$. We can prove this constructively and find the smallest example for this. Let $\ranking[r] = (1, 1/2, 1/3)$, $\ranking[s] = (1 - 2\varepsilon, 1 - \varepsilon, 1)$ and $\ranking[t] = (\alpha, 1, \alpha^2)$. One can verify that
\begin{eqnarray*}
    \lim_{\varepsilon \searrow 0} \fuji{}(\ranking[r], \ranking[s], 1) = \frac{2}{3} &\text{and}& \lim_{\varepsilon \searrow 0} \fuji{}(\ranking[r], \ranking[s], 2) = \frac{8}{9}\\
    \lim_{\alpha \nearrow \infty} \fuji{}(\ranking[r], \ranking[t], 1) = \frac{1}{6} &\text{and}& \lim_{\alpha  \nearrow \infty} \fuji{}(\ranking[r], \ranking[t], 2) = \frac{5}{9}\\
\end{eqnarray*}
therefore there exist $\varepsilon$ small enough and $\alpha$ big enough, such that assumptions i)--iii) hold.

{\bf The lower bound for AUC.}
Given Proposition \ref{pro:jaccard:min}, it is clear that the minimal $\auc{}$ value of a \jaccard{} curve exists. By explicitly computing the values $j_k$ from the corresponding proof and applying formula \eqref{eqn:auc}, we obtain
$$
\min_{\ranking[s]}\auc{}(\ranking[r], \ranking[s]) = \frac{1}{4(n - 1)}\cdot
\begin{cases}
(n^2 + 1) / n &;\; n\text{ is odd}\\
\hphantom{iiiiii}n &;\; n\text{ is even}
\end{cases}
$$
which converges to $1/4$ when $n\to\infty$.

In the above computation, a crucial point is that \jaccard{} ignores the actual (relevance) ranking scores. The case of \fuji{} taking into account these scores brings us to the following result.

\begin{proposition}\label{pro:fuji:argmin}
In the case of \fuji{}, the minimizing pair of rankings does not exist for any $k < n$.
\end{proposition}
Consequently, this also holds for the $\auc{}_\fuji{}$.

{\bf \jaccard{} is the exact lower bound for \fuji{}.}
The proof of the previous proposition leads to the final result, showing that \jaccard{} is only a limit case of \fuji{} (also similarly for $\auc{}_{\jaccard{}}$ and $\auc{}_{\fuji{}}$):
\begin{theorem}
    For every $\varepsilon >0$ and any two orderings of items $x_i$ (defined by two permutations $\pi$ and $\tau$),
    there exist rankings $\ranking[r] = (r_1, \dots, r_n)$ and $\ranking[s] = (s_1, \dots, s_n)$ with the following properties:
    i) $\ranking[r]$ and $\ranking[s]$ respectively induce the same ordering of items as $\pi$ and $\tau$,
    and ii) $|\fuji{}(\ranking[r], \ranking[s], k) - \jaccard{}(\ranking[r], \ranking[s], k)| \leq \varepsilon$, for $1\leq k \leq n$.
\end{theorem}

The theorem states that, given any two orderings of items, one can find two vectors of ranking scores that respect these two orders and for which, the points of \fuji{} curve lie arbitrarily close to the corresponding points of \jaccard{} curve. This happens
when the ratios $r_{(i)} / r_{(i + 1)}$ and $s_{(i)} / s_{(i + 1)}$ tend to go to infinity. In this case, \fuji{} approaches \jaccard{} from above,
since we always have $\fuji{}(\ranking[r], \ranking[s], k) \geq \jaccard{}(\ranking[r], \ranking[s], k)$.

\subsection{Correction for Chance}\label{sec:correction-for-chance}

Even though \jaccard{} as defined in Eq.~\ref{eqn:jaccard} does not possess this property, it can be easily adapted by computing the expected value $e(k) = \mathbb{E}_{\pi}[\jaccard{}(\ranking{}, \pi(\ranking{}), k)]$ where the distribution of permutation $\pi$ is uniform, and ranking $\ranking{}$ is arbitrary, since the scores are not taken into account directly.
For larger values of $n$, the exact computation may take some time,
but one can use $\hat{e}(k) = k / (2 n - k)$ instead which accurately approximates $e(k)$.

Normalizing \fuji{} in the same manner would demand computing 
\begin{equation*}
\mathbb{E}_{\ranking{}, \ranking[s]{}, \pi}[\jaccard{}(\ranking{}, \pi(\ranking[s]{}), k)].
\end{equation*}
This implies knowing the distribution of the scores, since $\ranking{}$ and $\ranking[s]{}$ cannot be arbitrary.
When $n$ is large, the distribution is expected to be long-tailed (e.g., a power-law distribution),
i.e., only a small proportion of items would be relevant. However, the exact distribution should be determined separately per use case.

\subsection{Utility of Ranking Scores}
Given Equations \eqref{eqn:fuzzy-mu} and \eqref{eqn:fuji}, it is evident that \fuji{} takes the ranking scores into account.

\section{Applications}

In this section, we demonstrate the capabilities of \fuji{} on the task of comparing feature ranking lists. More specifically, we consider feature ranking tasks on 24 classification datasets where the values of the categorical target variable $y$ depend on the values of numeric or categorical features $x_i$:
\texttt{APS failure} \cite{aps},
\texttt{biodegradability} \cite{biodeg-p2}, 
\texttt{bladder} \cite{dyrskjot2003identifying},
\texttt{childhood} \cite{cheok2003treatment},
\texttt{cmlTreatment} \cite{crossman2005chronic},
\texttt{coil2000} \cite{van2004bias},
\texttt{colon-cancer} \cite{alon1999broad},
\texttt{digits} \cite{xu1992methods},
\texttt{diversity} \cite{dvzeroski1998machine}, 
\texttt{dlbcl} \cite{dlbcl},
\texttt{gas drift} \cite{vergara2012chemical},
\texttt{genes} \cite{weinstein2013cancer},
\texttt{leukemia} \cite{golub1999molecular},
\texttt{madelon} \cite{guyon2008feature},
\texttt{mll} \cite{armstrong2002mll},
\texttt{optdigits} \cite{optdigits},
\texttt{OVA-Breast} \cite{stiglic2010stability},
\texttt{p-gp} \cite{levatic2013accurate},
\texttt{p53} \cite{danziger2009predicting},
\texttt{pd-speech} \cite{sakar2013collection},
\texttt{QSAR degradation} \cite{mansouri2013quantitative},
 \texttt{sonar} \cite{gorman1988analysis},
\texttt{srbct} \cite{khan2001classification}, and
\texttt{water-all} \cite{dvzeroski1998machine}. The datasets cover different domains, including: medicine (tumor analysis with gene expression data), biodegradability of chemicals, failure prediction, handwriting recognition, etc. A summary of their characteristics (in terms of the numbers of features and examples, and a brief description) is given in the supplementary material (Appendix D). 

We compare \fuji{} with three other similarity scores: i) \jaccard{}, since it is the base of \fuji{} ; ii) correlation, since it takes ranking scores into account and iii) Hamming distance, since it exhibits similar performance to \fuji{}, as seen in Sec.~\ref{sec:examples}. The considered ordered lists, correspond to different feature rankings produced by the following standard, and widely used, algorithms (their parameters are given in the parentheses):
Relief \citep{Kononenko2003} (15 neighbors of each instance are computed), Mutual Information (MI) \cite{Kraskov2004} (parameter-less), as well as two ensemble-based ranking scores Genie3 \cite{Huynh-Thu2010} and Random Forest (RF) \cite{Breiman2001} (each computed from a random forest ensemble with 200 trees of unlimited depth and feature subset size of $\sqrt{n}$).

\begin{figure}[!b]
\centering
\includegraphics[width=0.6\textwidth]{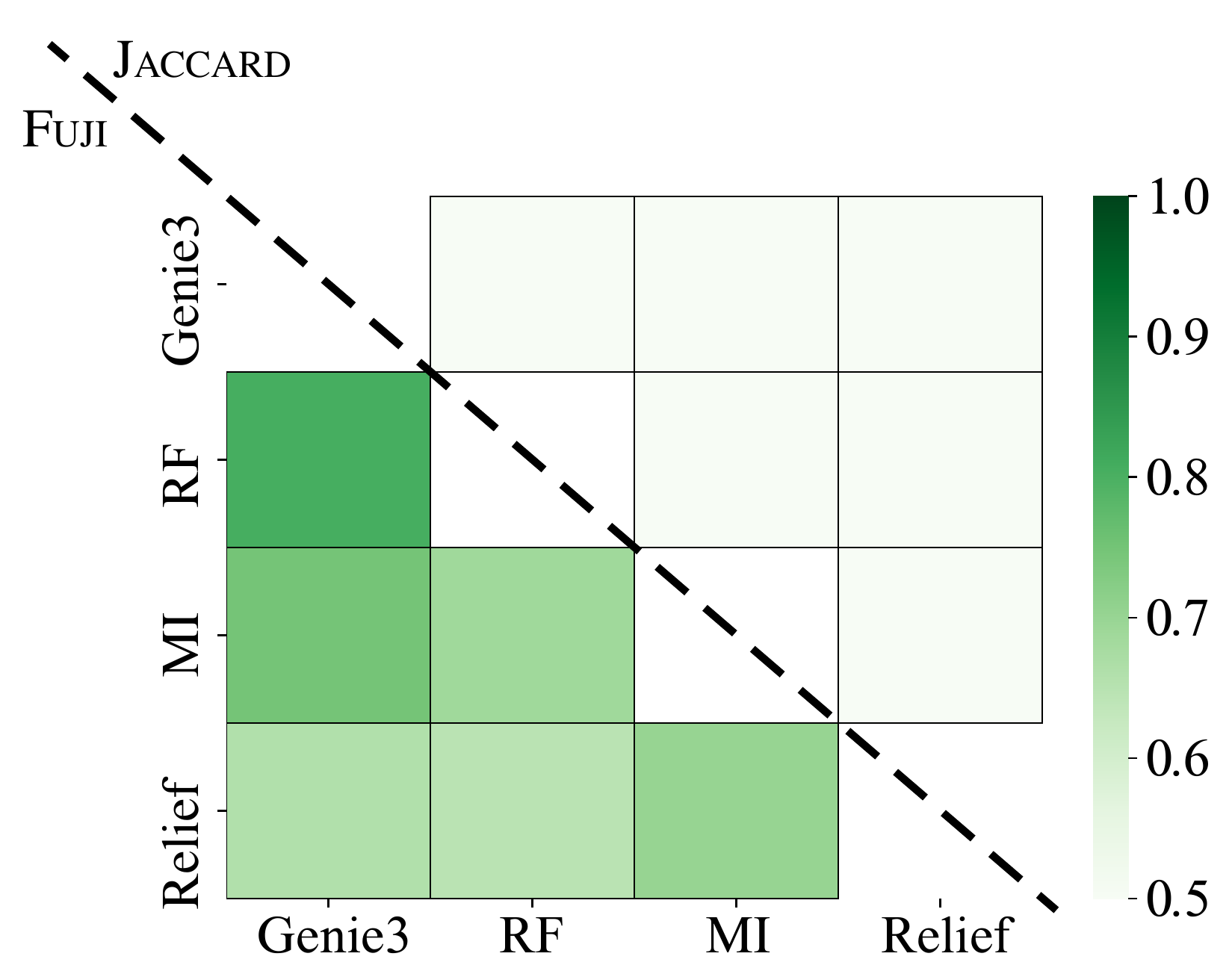}
\caption{Comparison of the rankings via the $\auc{}$ under the \fuji{} and \jaccard{} curve, averaged over the 24 datasets.}
\label{fig:stc:heatplot}
\end{figure}



We start our analysis with the average \fuji{} similarities across all datasets. The lower-left triangle in Fig.~\ref{fig:stc:heatplot} presents the average $\auc{}_\fuji{}$, while the upper-right triangle depicts the average $\auc{}_\jaccard{}$. The curves consist of points $(k, f_k)$, for $k\in\{1, 2, 4, \dots, 2^m, n\}$ where $m = \mathit{floor}(\log_2 n)$, in order to put more weight on the beginning of the ranking, which is typically of the highest interest. The figure clearly shows that \fuji{} can detect similarities between the compared rankings, which is not possible in the case of \jaccard{}. In particular, we can observe that the two ensemble-based rankings are the most similar on average. This makes sense since these two rankings are computed from the same underlying ensemble classifier. The least similar are the RF and Relief rankings, i.e., the one that uses an underlying classifier the most extensively, and the one that does not use it all. The level of similarity for the remaining five pairs of rankings is approximately the same. In contrast, we can see that \jaccard{} is practically unable to distinguish between different levels of similarity.

Next, we present the results for individual classification tasks. For brevity, here we only show tasks concerning high-dimensional datasets, i.e. the \texttt{genes} dataset with 20531 features (Fig.~\ref{fig:stc:genes}) and the \texttt{p53} dataset with 5408 features (Fig.~\ref{fig:stc:p53}). The remaining results of the other classification tasks are given in Appendix D of the supplementary material and are in line with the conclusions presented here. 

The comparison of the ensemble-based feature rankings on the \texttt{genes} dataset are presented in Fig.~\ref{fig:stc:genes:ensembles}. Comparing the RF and Genie3 ranking already shows the typical instability and pessimistic values of \jaccard{} curves:
since the first features of the rankings are different, the curve starts at $0$, quickly increases to $0.5$ at $k = 4$ and drops again to $0.15$ at $k = 15$. \fuji{} curve, on the other hand, correctly estimates that the rankings are not that different. For example, the top-ranked feature $x_{(1)}$ in the Genie3 ranking, has a rank of $5$ in the RF ranking -- \fuji{} can detect this by taking the relevance scores into the account.

\begin{figure*}[b!]
\centering
\begingroup
    \captionsetup[subfigure]{width=0.7\textwidth}
    \subfloat[\label{fig:stc:genes:ensembles} Genie3 and RF rankings on \texttt{genes} dataset]{
   \includegraphics[width=0.8\textwidth]{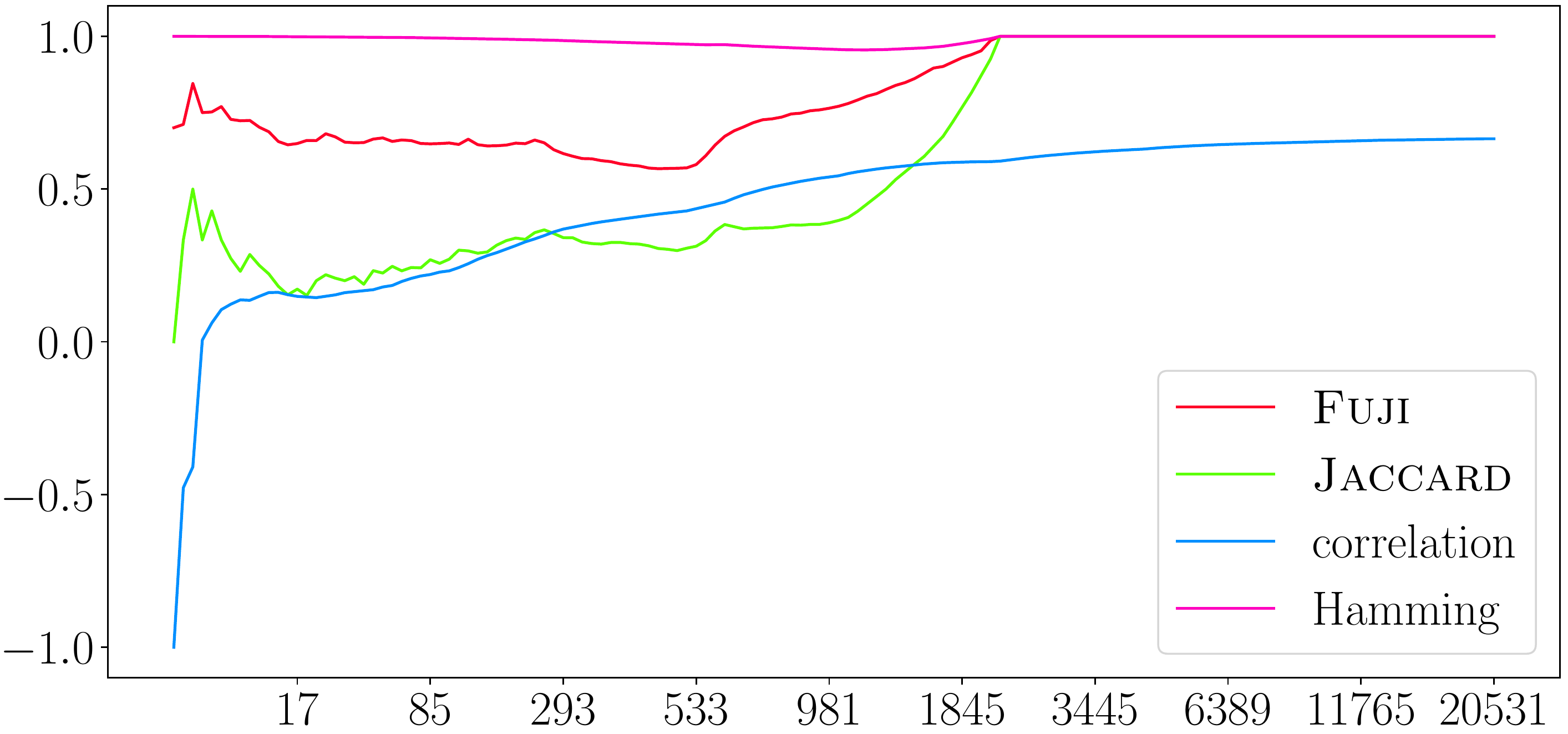}}
\endgroup
\begingroup
    \centering
    \captionsetup[subfigure]{width=0.7\textwidth}
    \subfloat[\label{fig:stc:genes:genie-mi-relief} Relief and MI rankings on \texttt{genes} dataset]{
    \includegraphics[width=0.8\textwidth]{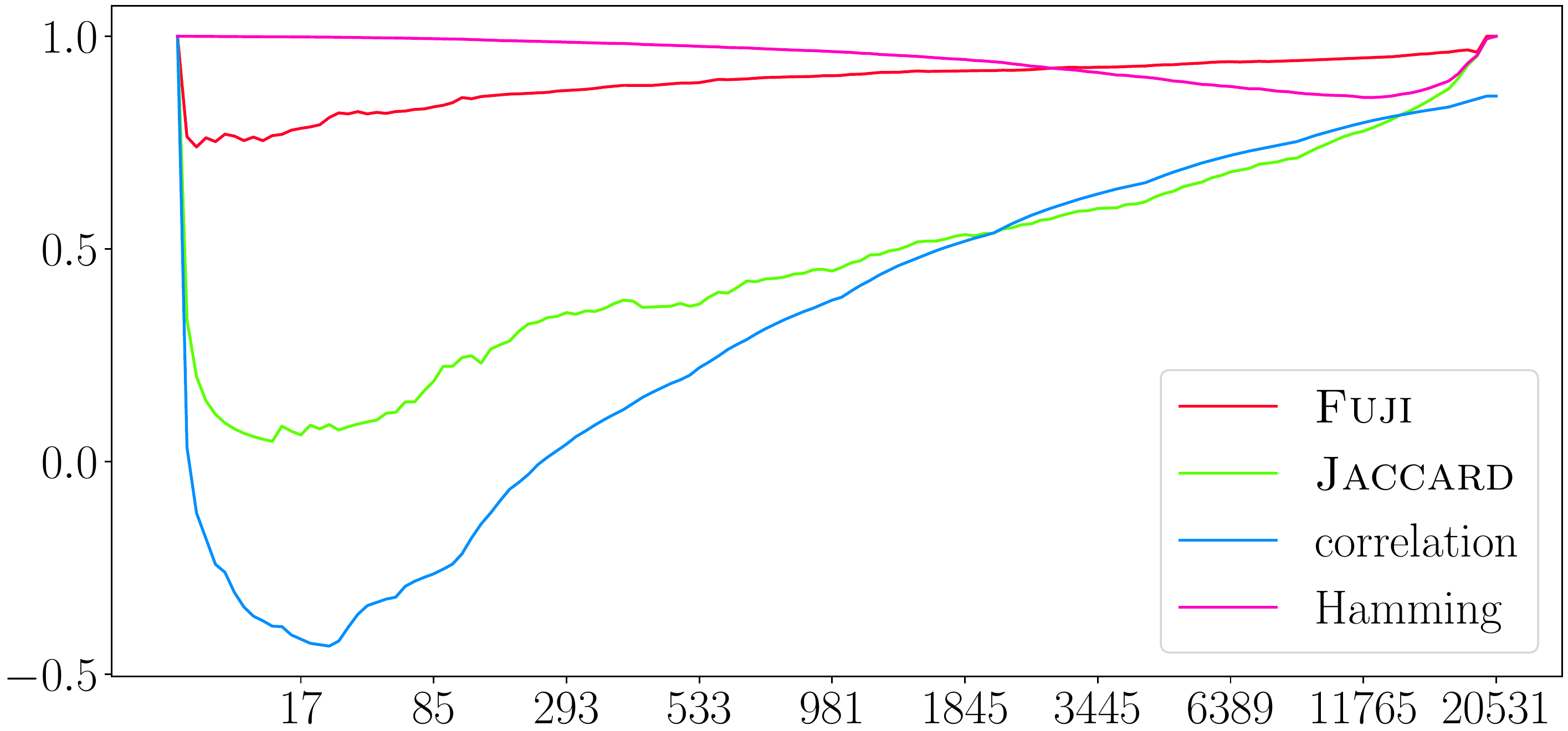}}
\endgroup
\caption{Similarity of the rankings for the \texttt{genes} dataset, for different numbers $k$ of top-ranked features (x-axis).
We compare the ensemble-based rankings (left), and Relief and MI rankings (right).}
\label{fig:stc:genes}
\end{figure*}

The relevance scores are taken into account also by correlation, however, we can see that this measure is too strict since it measures the similarity of the sets via linear dependence of the corresponding scores. Hamming similarity, on the other hand, is always too optimistic at the beginning of the rankings if the number of features $n$ is large, since it is based on the symmetric difference of the sets $\topset{r}{k}$ and $\topset{s}{k}$ which is always small (compared to $n$) for the lower values of $k$, no matter how different those sets are. After the point $k = 2117$, the curves of \fuji{}, \jaccard{} and Hamming coincide. Examination of the actual feature relevance scores reveals that all curves meet at the same point because both ranking algorithms consider the same majority of the features (about $~18000$) irrelevant (i.e., their relevance score is $0$). This is not surprising since these rankings are computed from the same ensemble model.

The "pessimistic" \jaccard{} estimates are again observed when comparing the MI and Relief rankings in Fig.~\ref{fig:stc:genes:genie-mi-relief}: Starting at $1$ due to the identical top-ranked feature of the rankings,
the \jaccard{} curve decreases abruptly until the next feature is added to the intersection of top-sets at $k = 12$. This is not the case for \fuji{}, which detects this feature already at $k = 4$ (when it is included in the Relief top-set). At that very same point, its MI relevance ($0.78$) is still similar enough to the relevance of $x_{(4)}$ in MI ranking ($0.83$).

\begin{figure*}[b!]
\centering
\begingroup
    \captionsetup[subfigure]{width=0.47\textwidth}
    \subfloat[\label{fig:stc:p53:ensembles}Genie3 and RF rankings on \texttt{p53} dataset]{
    \includegraphics[width=0.8\textwidth]{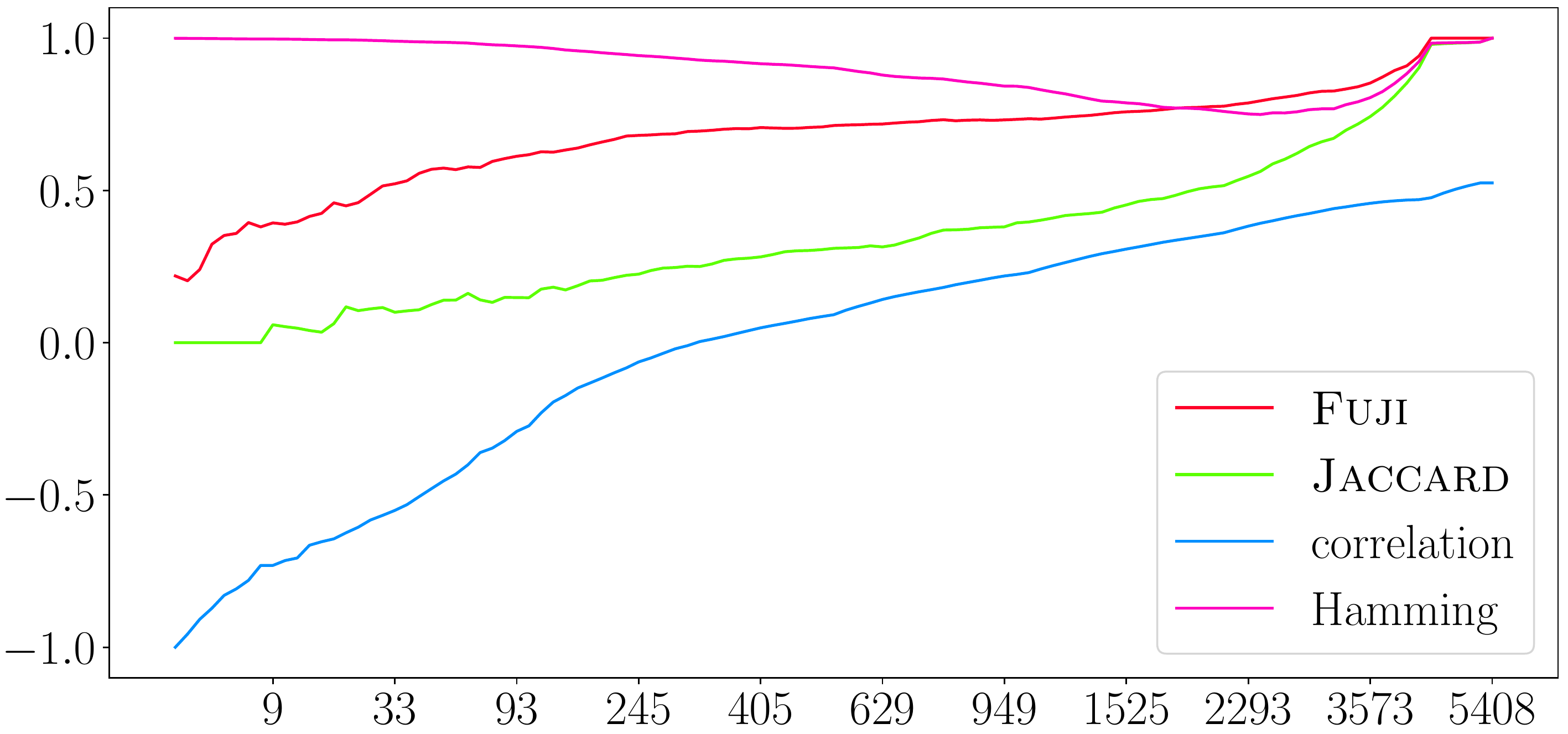}}
\endgroup
\begingroup
    \centering
    \captionsetup[subfigure]{width=0.47\textwidth}
    \subfloat[\label{fig:stc:p53:genie-mi-relief} Relief and MI rankings on \texttt{p53} dataset]{
    \includegraphics[width=0.8\textwidth]{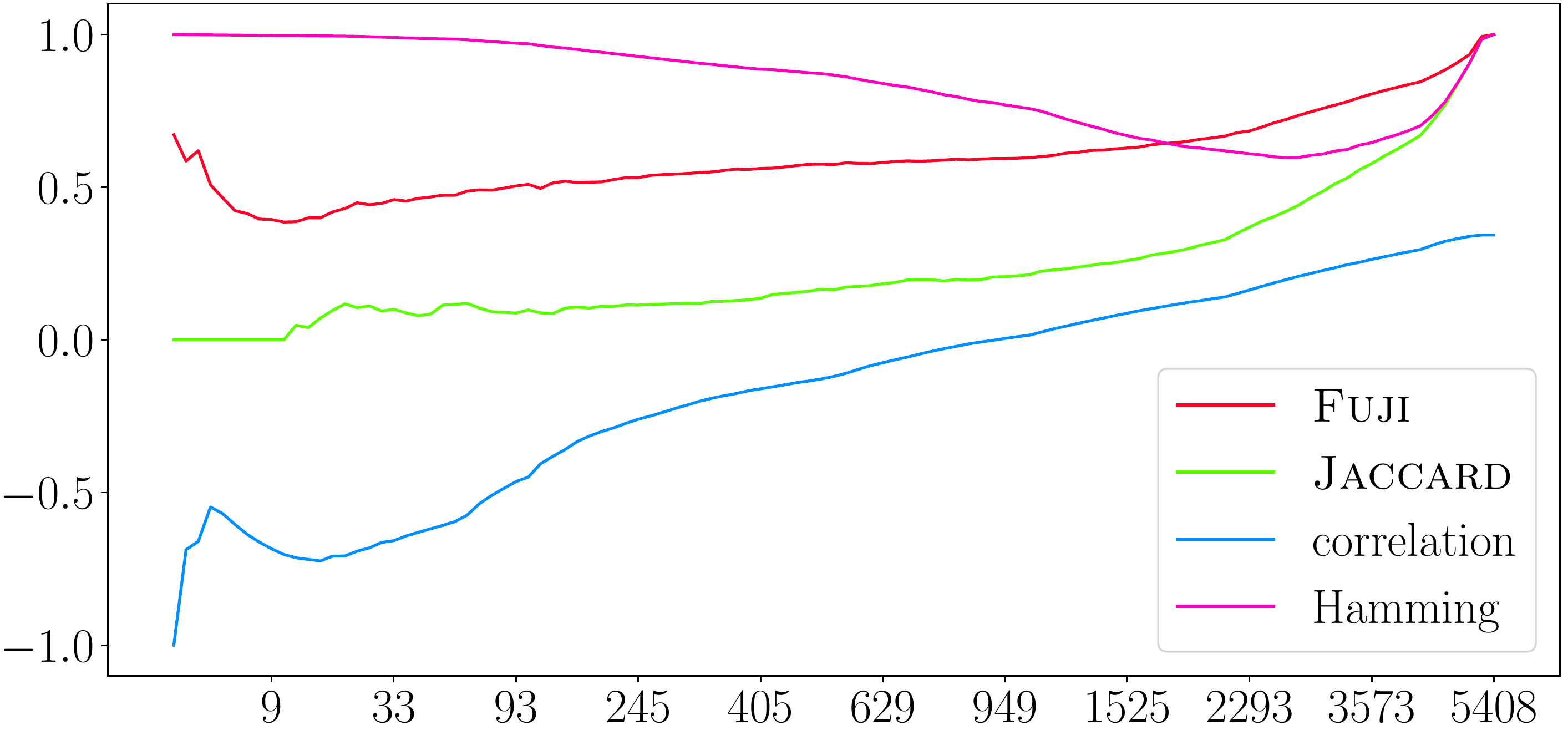}}
\endgroup

\caption{Similarity of the rankings for the \texttt{p53} dataset, for different numbers $k$ of top-ranked features (x-axis).
We compare the ensemble-based rankings (left), and Relief and MI rankings (right).}
\label{fig:stc:p53}
\end{figure*}

Finally, we present the analysis on the \texttt{p53} dataset. Similarly, we start with the comparison of the ensemble-based rankings (Fig.~\ref{fig:stc:p53:ensembles}). We can see that the Genie3 and RF ranking for this dataset are substantially less similar as compared to the \texttt{genes} dataset. For instance, the top-ranked feature identified by the RF-ranking has ten times lesser ranking score in Genie3-ranking than its top-ranked feature. This is well reflected in \fuji{} and \jaccard{} curves. Similarly, to \texttt{genes} dataset, this also holds when comparing Relief and MI rankings in Fig.~\ref{fig:stc:p53:genie-mi-relief}. In the case of Hamming and correlation, we cannot draw any firm conclusions from the curves.

\subsection*{A case study on The Cancer Genome Atlas dataset (TCGA)}

To further justify the utility of \fuji{}, we study the relationship between \fuji{} and \jaccard{} with respect to the identified top-ranked features. In particular, we examine them on the \texttt{genes} dataset, which given the results presented in Figure~\ref{fig:stc:genes:genie-mi-relief}, their scores differ substantially. The \texttt{genes} dataset consists of gene expression information for more than 20{,}000 genes, therefore we additionally inspect the top-ranked genes in more detail. More specifically, we individually focus on the 10 top-ranked genes by using the UniProt knowledge-base \cite{uniprot2017uniprot} and present this analysis in Table~\ref{tbl:comparison} along with the corresponding functional annotation. 

It can be observed that the most relevant gene is in both cases \textit{HNF1A}, a well-known oncogene (as expected for this dataset). However, the remainder of the ranked genes differs when considering \jaccard{} as opposed to \fuji{}. A more detailed inspection of the remaining nine genes reveals most of them are related to DNA-binding processes or ion transport pathways. For instance, genes such as \textit{HK3}, \textit{HNF1B} and \textit{CLRN2}, as identified by Relief, have similar functional analogues in \textit{GAK}, \textit{TRPM8}, \textit{EFHC2} and \textit{EMX1} as identified by Mutual Information-based ranking. More importantly, by inspecting the remaining GO terms, all of the observed genes correspond to a similar subset of metabolic pathways, all related to DNA, ATP or major signaling processes. This comparison implies that the two rankings should be similar, which is the case for \fuji{}, but not for \jaccard{}.

\begin{table*}[h] 
\caption{Comparison of highly ranked genes by MI and Relief with corresponding molecular function annotation (if available).}
\resizebox{\textwidth}{!}{
\begin{tabular}{l|lll|lll}
\hline
Rank & MI    &     Gene name     & Description (GO)                             & Relief &   Gene name    & Description (GO)                                \\ \hline
1    & 7964  & \textit{HNF1A}    & DNA binding                                  & 7964   & \textit{HNF1A} & DNA binding                                               \\
2    & 18381 & \textit{TMPRSS15} & scavenger receptor activity                  & 19236  & \textit{USF2}  & DNA-binding transfription                 \\
   &   &  & & & & factor activity \\
3    & 17109 & \textit{SPCS3}    & peptidase activity                           & 4773   & \textit{DDX3X} & ATPase activity                                           \\
4    & 5407  & \textit{EFHC2}    & calcium ion binding                          & 7896   & \textit{HK3}   & ATP binding                                               \\
5    & 18746 & \textit{TRPM8}    & calcium channel activity                     & 3921   & \textit{CLRN2} & calcium ion binding, cell adhesion                        \\
6    & 6816  & \textit{GAK}      & ATP binding                                  & 7965   & \textit{HNF1B} & DNA binding                                               \\
7    & 8349  & \textit{IGSF9}    & Cell-cell adhesion mediation                 & 8891   & \textit{KDM5C} & DNA binding, dioxygenase activity                         \\
8    & 7992  & \textit{HOMER1}   & G protein-coupled glutamate & 742    & \textit{ANKS3} & Unknown, potentially vasopressin \\
   &   &  & receptor binding  & & & signaling in the kidney \\
9    & 5576  & \textit{EMX1}     & DNA binding transcription    & 15301  & \textit{RPS4X} & RNA binding                                               \\
   &   &  & factor activity  & & & \\
10   & 6593  & \textit{FOS}      & chromatin binding                            & 19313  & \textit{UTS2}  & signaling receptor binding                \\ \hline              
\end{tabular}
}
\label{tbl:comparison} 
\end{table*}

\section{Conclusion}

In this study, we propose \fuji{} -- a novel approach for estimating the similarity between ordered lists. In particular, we showcased the theoretical properties of \fuji{}, proving its continuous-like nature, which renders it more robust than its counterpart -- the \jaccard{} index. Moreover, we empirically show the advantage of using \fuji{} over \jaccard{} as well as over ten other standard benchmark scores on three synthetic scenarios. Finally, we demonstrate its utility in a variety of real-world classification tasks.

While in this study we limit our focus to only estimating the similarity of feature ranking sets, the utility of \fuji{} is much broader. For instance, in the context of decision trees, it can be used for measuring tree similarity by evaluating the splits at each node. In more general terms, \fuji{} can be also readily employed for comparing machine learning methods (in general) across multiple performance metrics. 

Another application of \fuji{} relates to the comparison of multiple node centrality measures in the field of complex network analysis. Given the set nodes in a network, centralities pinpoint relevant nodes. There are many existing centralities, and it is not entirely clear in what relation they emerge on real-world networks. Moreover, recent advances in natural language processing exploit the notion of neural attention \cite{vaswani2017attention}. This mechanism yields real-valued scores for each token, which we believe is a suitable scenario for \fuji{}. One of the possible applications relates to bias detection in contemporary language models. For instance, different fine-tuning scenarios may yield different attention vectors for a given sentence. \fuji{} could also be used to detect and group distinct fine-tuning settings based on the similarity of the resulting attention vectors. In practice, apparent biases, e.g., political, could emphasize different tokens leading to differences when compared with \fuji{}.



\section*{Acknowledgements}
The computational experiments presented here were executed on a computing infrastructure from the Slovenian Grid (SLING) initiative, and we thank the administrators Barbara Kra\v{s}ovec and Janez Srakar for their assistance.

\section*{Funding}
This work was supported by the Slovenian Research Agency via the grant P2-0103 and a young researcher grant to MP and B\v{S}. DK and NS also acknowledge the support of the Slovenian Research Agency via grant J2-9230.


\bibliography{main}

\newpage
\appendix

\section{Proofs}

For completeness, we again give the equations that are present in the main document and are also referenced here.

\begin{equation}
    \label{app:eqn:jaccard}
    \jaccard{}(\ranking[], \ranking[s], k) = \left| \topset{r}{k} \cap \topset{s}{k}\right| / \left| \topset{r}{k} \cup \topset{s}{k} \right|
\end{equation}

\begin{equation}
    \label{app:eqn:fuzzy-mu}
    \mu_{\topset{}{k}}^{F}(x_i)= 
    \begin{cases}
    \hphantom{ii.}1 &;\; x_i \in \topset{}{k} \\
    \score{}_i / \score{}_{(k)} &;\; x_i \notin \topset{}{k} \land \score{}_{(k)} > 0 \\
       \hphantom{ii.}0 &;\; \text{otherwise}
    \end{cases}
\end{equation}

\begin{equation}
    \label{app:eqn:auc}
    \auc{}_{\fuji{}}(\ranking[r], \ranking[s]) = \frac{1}{n - 1}\left(\frac{f_1 + f_n}{2} + \sum_{k = 2}^{n - 1} f_k \right)\text{.}
\end{equation}

\begin{proposition}
\label{app:pro:jaccard:min}
When using \jaccard{}, the least similar ranking to a given ranking $\ranking{}$ is every ranking $\ranking[s]$,
such that $\rank{s}(x_i)  = n + 1 - \rank{}(x_i)$. 
\end{proposition}

\begin{proof}
Without loss of generality, we assume that $\rank{}(x_i) = i$, for all $i$.
Let $\ranking[s']$ be an arbitrary ranking such that $\rank{s'}(x_1) = k_0 \neq n$, and let $x_\ell$ be the item with rank $n$.
Let $\ranking[s]$ be the ranking that is obtained from $\ranking[s']$ by exchanging the ranks of items $x_1$ and $x_\ell$.
We show that ranking $\ranking[]$ is then more similar to $\ranking[s]$ than to ranking $\ranking[s']$.
We compare the intersection and union sizes from \ref{app:eqn:jaccard} in both cases, so let be
$j_k' = \jaccard{}(\ranking[r], \ranking[s'], k) = i_k' / u_k'$ and $j = \jaccard{}(\ranking[r], \ranking[s], k) = i_k / u_k$.
It holds that
\begin{equation}
    \label{app:eqn:invariant}
    \begin{cases}
    u_k' = u_k\; \land\; i_k' = i_k &;\; k < k_0\\
    u_k' \leq u_k\; \land\; i_k'\geq i_k &;\; k \geq k_0
    \end{cases}\text{.} 
\end{equation}
The case $k < k_0$ is obvious since these parts of the rankings $\ranking[s]$ and $\ranking[s']$ coincide. For $k\geq k_0$, we know that i) moving $x_1$ to the last place
(in the ranking) does not decrease the union size $u_k'$ as compared to $u_k$ since $x_1\in \topset{r}{k}$, and ii) moving $x_\ell$ to the $k_0$-th place might increase some union sizes $u_k$ by $1$. Thus, $u_k' \leq u_k$. An analogous argument shows that $i_k'\geq i_k$: i) moving $x_1$ to the last place decreases the intersection size $i_k'$ by $1$, and ii) moving $x_\ell$ to the $k_0$-th place might result in increasing it back. Thus, $j_k'\geq j_k$. 

In the same manner, one proceeds to proving that $\rank{s}(x_i)  = n + 1 - \rank{}(x_i)$, for $i > 1$. Doing so, we only have to be more careful when proving Eq.~\eqref{app:eqn:invariant},
since, for example, if $i = 2$ and $\rank{s'}(x_2) = 1$, moving $x_2$ to the place $n - 1$ decreases $u_k'$ (but in that case, moving $x_\ell$ to the first place increases it back).
\end{proof}

\begin{proposition}\label{app:pro:fuji:min}
Let $S$ be the set of rankings generated by all permutations of the fixed scores $s_{(i)}$, and $\ranking{}$ an arbitrary ranking.
To the ranking $\ranking[]$ least similar element of $S$ is the
ranking $\ranking[s]$, such that $\rank{s}(x_i)  = n + 1 - \rank{}(x_i)$. 
\end{proposition}
As it shall be seen in the Proposition \ref{app:pro:fuji:argmin}, $\inf_{\ranking[s]} \auc{}_{\fuji{}}(\ranking[], \ranking[s])$ is not achieved so there is no \emph{least similar ranking}.
However, the upper claim can still be proved.

\begin{proof}
We repeat the argument from the proof of Proposition \ref{app:pro:jaccard:min} and also use the notation where $\ranking[r]$ shall be the ranking such that $\rank{r}(x_i) = i$, for all $i$. 
The ranks of the items in the rankings $\ranking[s]$ and $\ranking[s']$ are all equal except for the items $x_1$ and $x_\ell$ such that $\rank{s}(x_1) = \rank{s'}(x_\ell) = n$ and
$\rank{s'}(x_1) = \rank{s}(x_\ell) = k_0$.

Since replacing \jaccard{} by \fuji{} score does not change the union sizes, one has to check whether the inequality $i_k'\geq i_k$ still holds when $k\geq k_0$.
Let $U_k = \topset{r}{k}\cup \topset{s}{k}$ and $U_k' = \topset{r}{k}\cup \topset{s'}{k}$. Note that 
$$
i_k'\quad 
\begin{cases}
 = i_k &;\; x_1, x_\ell \in U_k' \\
 \geq i_k  &;\; x_1\in U_k'\; \land\; x_\ell\notin U_k'\\ 
 = i_k &;\; x_1, x_\ell \notin U_k'
\end{cases}
$$
by the similar arguments as in the proof of Proposition \ref{app:pro:jaccard:min} with the difference that we are not adding or subtracting only 0s and 1s but rather values of $\mu$ from Eq.~\eqref{app:eqn:fuzzy-mu}.
Note that only the upper three cases need to be taken into account since the case $x_1\notin U_k'\; \land\; x_\ell\in U_k'$ is impossible
by construction. Similarly, the proof can be completed by induction.
\end{proof}

\begin{proposition}\label{app:pro:fuji:argmin}
In the case of \fuji{}, the minimizing pair of rankings does not exist for any $k < n$.
\end{proposition}

\begin{proof}
If there is a minimizing pair of rankings $\ranking[r] = (r_1, \dots, r_n)$ and $\ranking[s] = (s_1, \dots, s_n)$,
from Proposition \ref{app:pro:fuji:min} follows that
$r_1 > r_2 > \cdots > r_n > 0$ and $0 < s_1 < s_2 < \cdots < s_n$. We prove the claim for even values of $n$, since the odd-$n$ case can be proven in the same manner.
For a more compact notation, we introduce the function $G$ with four arguments: ranking $\ranking[t]$ and three indices $a$, $b$ and $c$. It is defined as
$$
G(\ranking[t], a, b, c) = \frac{1}{t_a}\sum_{i = b}^{c} t_i\text{.}
$$
If we explicitly compute the scores $f_k = \fuji{}(\ranking[r], \ranking[s], k)$, for $1\leq k < n$, we obtain
\begin{equation*}
    \label{app:eqn:fuji:argmin}
    f_k =
    \begin{cases}
        \hphantom{iiii}\frac{G(\ranking[s], n - k + 1, 1, k) + G(\ranking[r], k, n - k + 1, n)}{2k} &;\; k\leq \frac{n}{2}\\
        \frac{2k - n + G(\ranking[s], n - k + 1, 1, n - k) + G(\ranking[r], k, k + 1, n)}{n} &;\; k > \frac{n}{2}
    \end{cases}\text{.}
\end{equation*}

Since $G(\ranking[t], a, b, c) > 0$, the goal is to show that all values of $G$ that appear above can be arbitrarily close to $0$.
Note first that every index $a$ appears at most once for each of the rankings $\ranking[r]$ and $\ranking[s]$. Therefore, we can choose $\varepsilon > 0$ and achieve 
$\varepsilon \geq G(\ranking[t], a, b, c)$ for all appearing combinations $(\ranking[t], a, b, c)$ by recursively defining the ranking scores $s_i$ and $r_i$ as follows.
First, we set $s_1 = 1$. Then, for $k = n - 1, \dots, n/2 + 1$, we express $s_{n - k + 1}$ from $\varepsilon = G(\ranking[s], n - k + 1, 1, n - k)$.
By doing so, it is assured that $s_1 < \cdots < s_{n / 2}$. For $k \leq n / 2$ one proceeds similarly, however, $\varepsilon = G(\ranking[s], n - k + 1, 1, k)$ now returns only a candidate value for $s_{n- k + 1}$.
This may need to to be increased, to assure $s_{n/2} < \cdots < s_{n}$. Similarly, we construct the values $r_i$.
\end{proof}

This proof gives us an idea that leads to the final result that shows that \jaccard{} is only a limit case of \fuji{}:
\begin{theorem}
    For every $\varepsilon >0$ and any two orderings of items $x_i$ (defined by two permutations $\pi$ and $\tau$),
    there exist rankings $\ranking[r] = (r_1, \dots, r_n)$ and $\ranking[s] = (s_1, \dots, s_n)$ with the following properties:
    i) $\ranking[r]$ and $\ranking[s]$ respectively induce the same ordering of items as $\pi$ and $\tau$,
    and ii) $|\jaccard{}(\ranking[r], \ranking[s], k) - \fuji{}(\ranking[r], \ranking[s], k)| \leq \varepsilon$, for $1\leq k \leq n$.
\end{theorem}

\begin{proof}
    \fuji{} will be arbitrarily close to \jaccard{} when the additional terms that appear in \fuji{} computation (such as the $G$ values from the previous proof)
    go to $0$. In order to achieve that, the \fuji{} membership function $\mu^F$ (Eq.~\eqref{app:eqn:fuzzy-mu}) should be close to $0$ for all items
    that are not in the intersection of the two top-ranked items sets. Suppose $x_{\pi(1)}$ and $x_{\tau(1)}$ are the top-ranked items in the rankings $\ranking[r]$ and $\ranking[s]$
    respectively.
    Then, this can be achieved, for example, by setting $r_{\pi(1)} = s_{\tau(1)} = 1$ and defining geometrically progressing scores $r_{\pi(i + 1)} = \alpha r_{\pi(i)}$ and $s_{\tau(i + 1)} = \alpha r_{\tau(i)}$
    for some $\alpha$ (depending on $\varepsilon$) small enough. 
\end{proof}

\section{Properties of Correlation}
We have to show that correlation is not \textbf{fully defined} and does not possess the \textbf{maximum} property,
but that it possesses the properties \textbf{bounded}, \textbf{correction for chance} and \textbf{considers rank scores}.

The last property is obvious, so we will only give comments on the others.

Since the correlation can be computed only between two lists of the same size, it is not fully defined. Its boundedness is also well known:
it takes values from the interval $[-1, 1]$.

Since the value $1$ is achieved if and only if the two samples at hand are positively linearly dependent,
the same order of the items in the lists is not sufficient for achieving the maximal value. Thus, correlation does not possess maximum property.

To show that it possesses the correction for chance, we take a similar approach to that when we discussed this property for \fuji{} in the Section Properties.

Let $c(\ranking{}, \ranking[s]{})$ denote the correlation between two rankings. We will compute the expected value
$\mathbb{E}_{\ranking{}, \ranking[s]{}, \pi}[c(\ranking{}, \pi(\ranking[s]{}))]$, where the rankings $\ranking{}$ and $\ranking[s]{}$ are 
such that $r_1 < r_2 <\cdots < r_n$ and $s_1 < s_2 <\cdots < s_n$.

Note that
$$
\mathbb{E}_{\ranking{}, \ranking[s]{}, \pi}[c(\ranking{}, \pi(\ranking[s]{}))] = \mathbb{E}_{\ranking{}, \ranking[s]{}}\left[\mathbb{E}_\pi[c(\ranking{}, \pi(\ranking[s]{})) \mid \ranking{}, \ranking[s]{}]\right].
$$
Since we do not know the distribution of the scores $\ranking{}$ and $\ranking[s]{}$, we rather show that the conditional expectations over permutations equal $0$.
It suffices to show that $$f(\ranking{}, \ranking[s]{}, \pi)  = \sum_\pi \sum_i (r_i - \bar{r})(s_{\pi(i)} - \bar{s}) = 0$$
where $\bar{r}$ and $\bar{s}$ are the average values of the corresponding scores. Indeed,
\begin{eqnarray*}
f(\ranking{}, \ranking[s]{}, \pi) &=& \sum_\pi \sum_i (r_i - \bar{r})(s_{\pi(i)} - \bar{s}) \\
                                  &=& \sum_i \sum_\pi (r_i - \bar{r})(s_{\pi(i)} - \bar{s}) \\
                                  &=& \sum_i (r_i - \bar{r}) \sum_\pi (s_{\pi(i)} - \bar{s}) \\
                                  &=& \sum_i (r_i - \bar{r}) \left(\left[(n - 1)! \sum_j s_j\right] - n! \cdot \bar{s} \right) \\
                                  &=& \sum_i (r_i - \bar{r}) \cdot 0\\
                                  &=& 0.
\end{eqnarray*}

\clearpage

\section{\fuji{} algorithm }

We first prove its correctness and then analyze its time complexity.

\begin{algorithm} 
    \caption{$\mathit{SimilarityCurve}(\ranking[r], \ranking[s])$}\label{app:alg:curves}
     \begin{algorithmic}[1]
        \STATE{$\bm{\mathit{xs}}_1 = $ items, sorted decreasingly w.r.t.~$\ranking[r]$}
        \STATE{$\bm{\mathit{xs}}_2 = $ items, sorted decreasingly w.r.t.~$\ranking[s]$}
        \STATE{$D = \emptyset$} \COMMENT{symmetric difference of $\topset{r_1}{k}$ and $\topset{r}{k}$}
        \STATE{$I = \emptyset$} \COMMENT{intersection of $\topset{r}{k}$ and $\topset{s}{k}$}
        \STATE{$\mathit{curve} = $ list of length $n$, initilized with $1$-s}
        \FOR{$k = 1, 2, \dots, n - 1$}
            \IF{$ r_{(k)} = s_{(k)} = 0$}\label{app:alg:curves:break}
                \STATE{{\bf break}}
            \ENDIF
            \FOR{$\bm{\mathit{xs}}\in \{\bm{\mathit{xs}}_1, \bm{\mathit{xs}}_2\}$}
                \STATE{$x = \bm{\mathit{xs}}[k]$}
                \IF{$x\notin D$}
                    \STATE{add $x$ to $D$}\label{app:alg:curves:sym-diff}
                \ELSE
                    \STATE{move $x$ from $D$ to $I$} \label{app:alg:curves:intersection}
                \ENDIF
            \ENDFOR
            \STATE{$\mathit{sizeI} = |I|$}
            \FOR{$x \in D$}\label{app:alg:curves:omitFirst}
                \STATE{$m_1 =$ compute $\mu_{\topset{r}{k}}(x)$ via Eq.~\eqref{app:eqn:fuzzy-mu}}
                \STATE{$m_2 =$ compute $\mu_{\topset{s}{k}}(x)$ via Eq.~\eqref{app:eqn:fuzzy-mu}}
                \STATE{$\mathit{sizeI}\, +\!= \min\{m_1, m_2\}$}\label{app:alg:curves:omitLast}
            \ENDFOR
            \STATE{$\mathit{curve}[k] = \mathit{sizeI} / (|I| + |D|)$}
        \ENDFOR
    \STATE{\textbf{return} curve}
    \end{algorithmic}
\end{algorithm}

\textbf{Correctness.} Given $\topset{r}{n} = \topset{s}{n}$ for any two rankings, $\fuji{}(\ranking[r], \ranking[s], n) = \jaccard{}(\ranking[r], \ranking[s],  n) = 1$. Thus, the outer for loop's upper bound can be $n - 1$. 

Moreover, one can interrupt the iteration even earlier (line 8). 
In particular, if there exists a $k$, where the lowest relevance score in both sets $\topset{r}{k}$ and $\topset{s}{k}$ (namely $r_{(k)}$ and $s_{(k)}$) is $0$, then also all the items
$x_i \notin \topset{r}{k}\cap\topset{s}{k}$ have score $0$. This means that the membership function from Eq.~\eqref{app:eqn:fuzzy-mu} will result in $1$ for all $x_i \in \topset{r}{k}\cup\topset{s}{k}$,
hence $\mathit{curve}[k] = 1$, yielding the rest of the computation redundant.

Otherwise, the algorithm proceeds as follows. For both ranking sets, one takes the $k$-th ranked item $x$ from both sets. If $x$ has not been considered before in the previous iterations, i.e., have not been ranked better by any of the rankings, is assigned to the symmetric difference $D$ of the sets $\topset{r}{k}$ and $\topset{s}{k}$ (line 
13)
Otherwise, since it has been previously assigned $x\in D$, it is moved
to the intersection $I$ of these two sets (line 15). 

After updating the sets $D$ and $I$, one proceeds to computing the size of $\topset{r}{k}\cap\topset{s}{k}$ and $\topset{r}{k}\cup\topset{s}{k}$. Since $\topset{r}{k}\cup\topset{s}{k}$ is a disjunctive union of $D$ and $I$,
the latter is computed as $|\topset{r}{k}\cup\topset{s}{k}| = |D| + |I|$. To compute the final score, one has to first compute how much do the items $x\in D$ additionally contribute to the \emph{sizeI} of the intersection $I$ (Eq.~\eqref{app:eqn:fuzzy-mu}). The correctness of the algorithm follows, since $\max \{\mu_{\topset{r}{k}}^F(x), \mu_{\topset{s}{k}}^F(x) \} = 1$.

Alg.~\ref{app:alg:curves} can be also utilized for computing the curve in the case of \jaccard{} by skipping the lines 19--23. 

\textbf{Time complexity.} The algorithm first sorts items according to the rankings at its input. This has $\mathcal{O}(n\log n)$ complexity. Computing the \fuji{} curves, in the worst-case, needs $\mathcal{O}(n^2)$ time. However, in practice, such scenarios can be avoided when the early stopping criterion in line 8 
applies, and by optimizing the computation of the intersection size and operating only with $I$ and $D$ rather than with $\topset{r}{k}$ and $\topset{s}{k}$ (thus avoiding to compute $\mu$ values equal to $1$). 

Moreover, additional optimization of the algorithm, without any particular loss of performance, can be also achieved by
computing the values $\mathit{curve}[k]$ only for a subset of $\{1, \dots, n - 1\}$ which is more dense at smaller values.
Such an optimization may even be preferred in practice, since one is mostly interested in the top-ranked items and typically scores such as  $\fuji{}(\ranking[r], \ranking[s], k)$ and $\fuji{}(\ranking[r], \ranking[s], k + 1)$ are quite similar for large $k$. Moreover, computing $\auc{}$ (Eq.~\eqref{app:eqn:auc}), while skipping some values of $k$, effectively gives higher weight to the top of the ranking.

\newpage
\section{Real-world experiments}

\begin{figure}[!htb]
\centering
\includegraphics[width=0.5\textwidth]{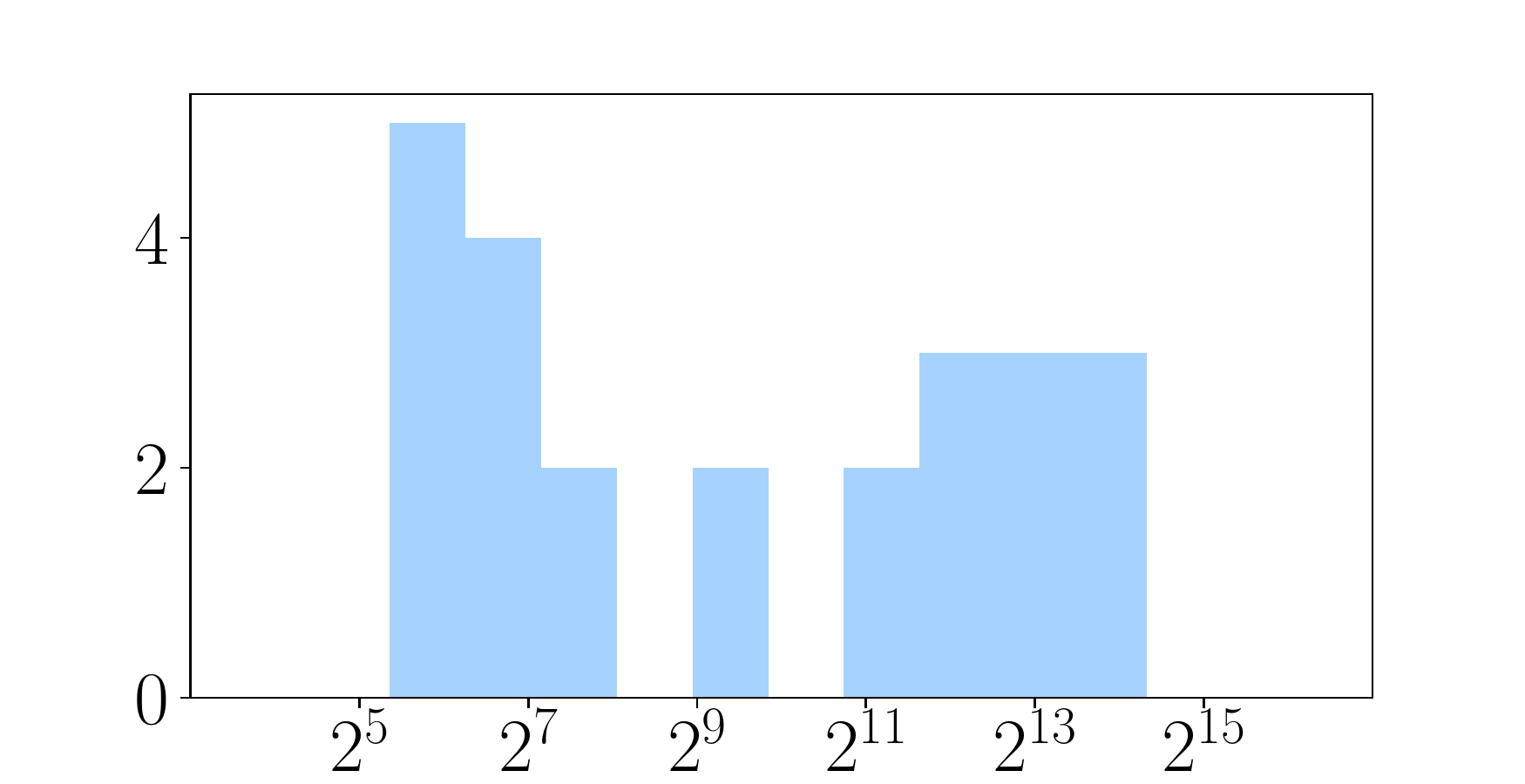}
\caption{Number of data sets in this study (y-axis) according to the number of features (x-axis).}
\label{app:fig:n-features}
\end{figure}

\begin{longtable}{l| r r |p{7cm}}
    \caption{Datasets used in this work.}
    \label{app:tab:data1}\\
    \hline
    Name & Features & Examples & Description  \\ \hline
        \hline \hline
APS failure  \cite{aps}                          & 171  & 76000& APS failure prediction \\
biodegredability \cite{biodeg-p2}           & 62   & 328  & Biodegradability of commercial compounds \\  
bladder \cite{dyrskjot2003identifying}           & 5725 & 40   & Identifying distinct classes of bladder carcinoma using microarrays \\
childhood  \cite{cheok2003treatment}             & 8281 & 110  & Treatment-specific changes in gene expression discriminate in vivo drug response in human leukemia cells. \\
cmlTreatment  \cite{crossman2005chronic}         & 12626& 28   & Identifying patients with chronic myeloid leukemia that do not respond to standard imatinib treatment. \\
coil2000  \cite{van2004bias}                     & 86   & 9822 & The Insurance Company Case.  \\
colon-cancer  \cite{alon1999broad}               & 2001 & 62   & Broad patterns of gene expression revealed by clustering analysis of tumor and normal colon tissues probed by oligonucleotide arrays. \\
digits  \cite{xu1992methods}                     & 65   & 1797 & Hand-written digit recognition \\
diversity  \cite{dvzeroski1998machine}           & 87   & 292  & Machine learning applications in biological classification of river water quality \\
dlbcl  \cite{dlbcl}                              & 7071 & 77   & The gene-expression based distinguishing between Diffuse large B-cell lymphomas (DLBCL) and follicular lymphomas (FL). \\
gas drift  \cite{vergara2012chemical}            & 129  & 13910& Chemical gas sensor drift compensation using classifier ensembles, Sensors and Actuators \\
genes  \cite{weinstein2013cancer}                & 20532& 801  & This collection of data is part of the RNA-Seq (HiSeq) PANCAN dataset. \\
leukemia  \cite{golub1999molecular}              & 5148 & 72   & Molecular classification of cancer: class discovery and class prediction by gene expression monitoring. \\
madelon  \cite{guyon2008feature}                 & 501  & 2000 & Feature Extraction, Foundations and Applications. Studies in Fuzziness and Soft Computing \\
mll  \cite{armstrong2002mll}                     & 12534& 72   & MLL translocations specify a distinct gene expression profile that distinguishes a unique leukemia\\
optdigits  \cite{optdigits}                      & 63   & 5620 & Optical Recognition of Handwritten Digits \\
OVA-Breast  \cite{stiglic2010stability}          & 10937& 1545 & Stability of ranked gene lists in large microarray analysis studies \\
p-gp  \cite{levatic2013accurate}                 & 184  & 932  & P-gp drug recognition induced from a cancer cell line cytotoxicity screen.\\
p53  \cite{danziger2009predicting}               & 5409 & 31420& Predicting Positive p53 Cancer Rescue Regions Using Most Informative Positive (MIP) Active Learning \\
pd-speech  \cite{sakar2013collection}            & 754  & 756  & Parkinson disease prediction from speech features \\
QSAR degradation  \cite{mansouri2013quantitative}& 42   & 1055 &  Quantitative Structure - Activity Relationship models for ready biodegradability of chemicals \\
sonar  \cite{gorman1988analysis}                 & 61   & 208  & Analysis of Hidden Units in a Layered Network Trained to Classify Sonar Targets \\
srbct  \cite{khan2001classification}             & 2309 & 83   & Classification and diagnostic prediction of cancers using gene expression profiling and artificial neural networks \\
water-all  \cite{dvzeroski1998machine}           & 81   & 292  & Machine learning applications in biological classification of river water quality \\
\end{longtable}

\clearpage
\figureBlock{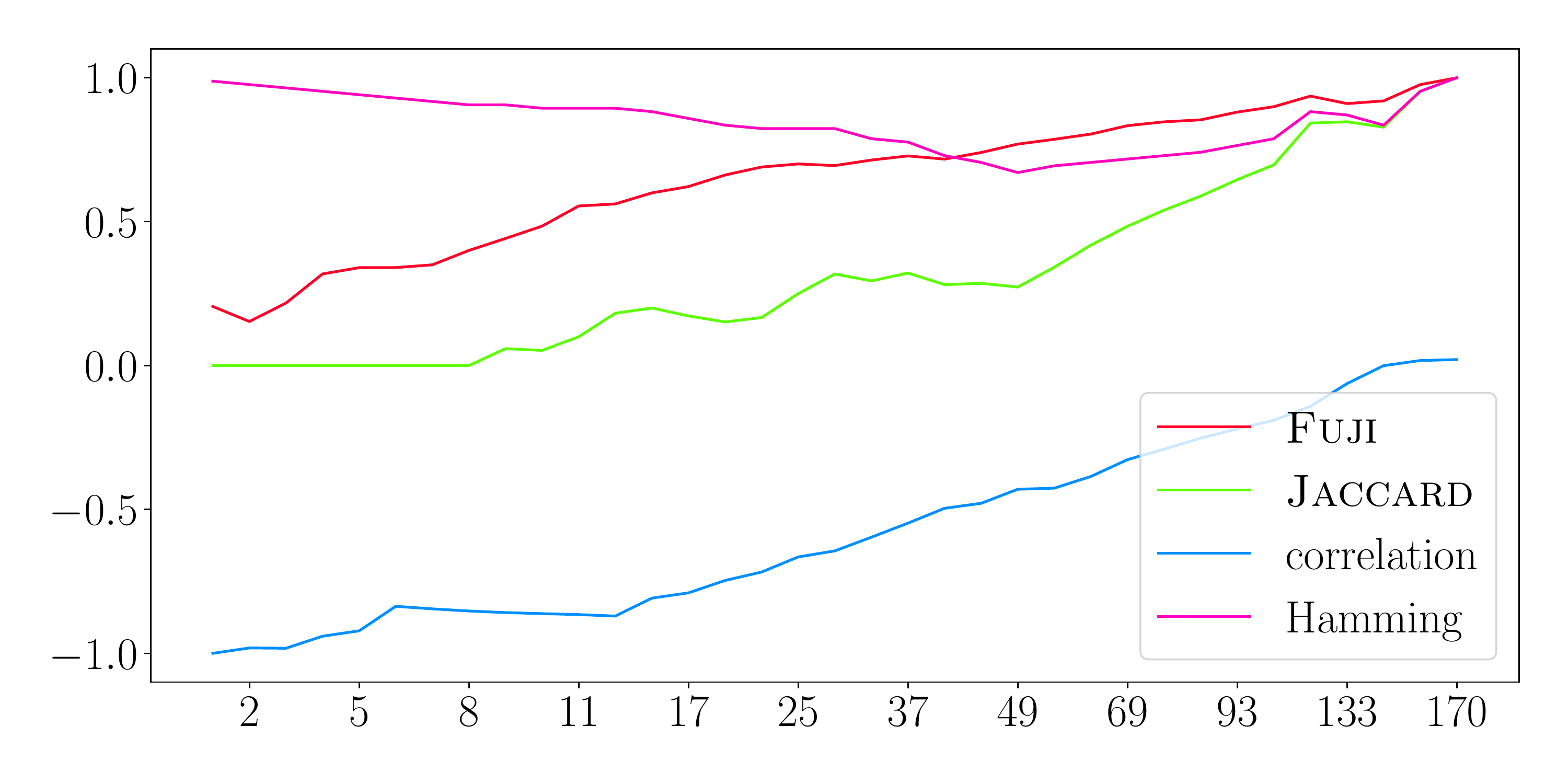}
\figureBlock{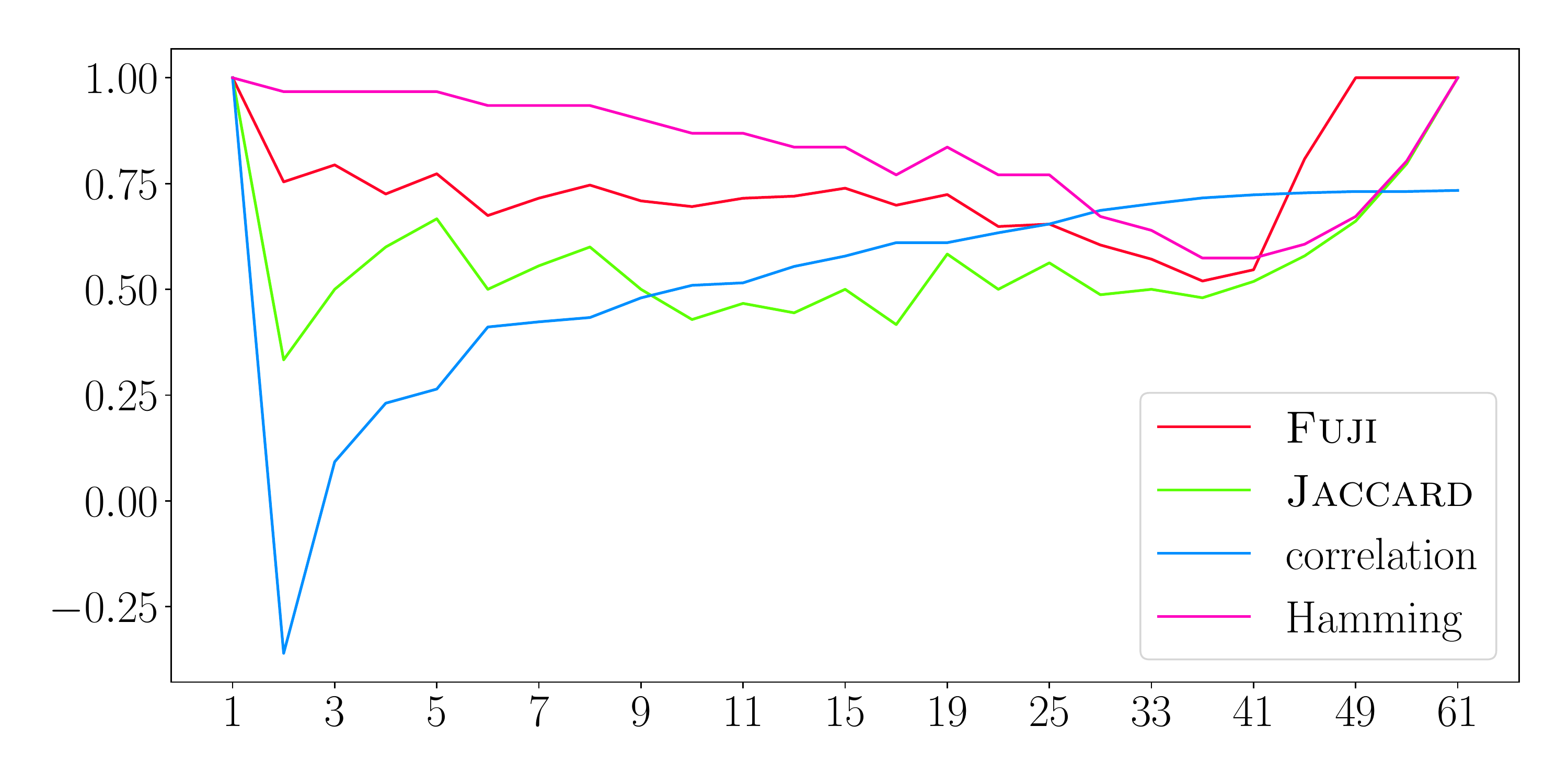}
\figureBlock{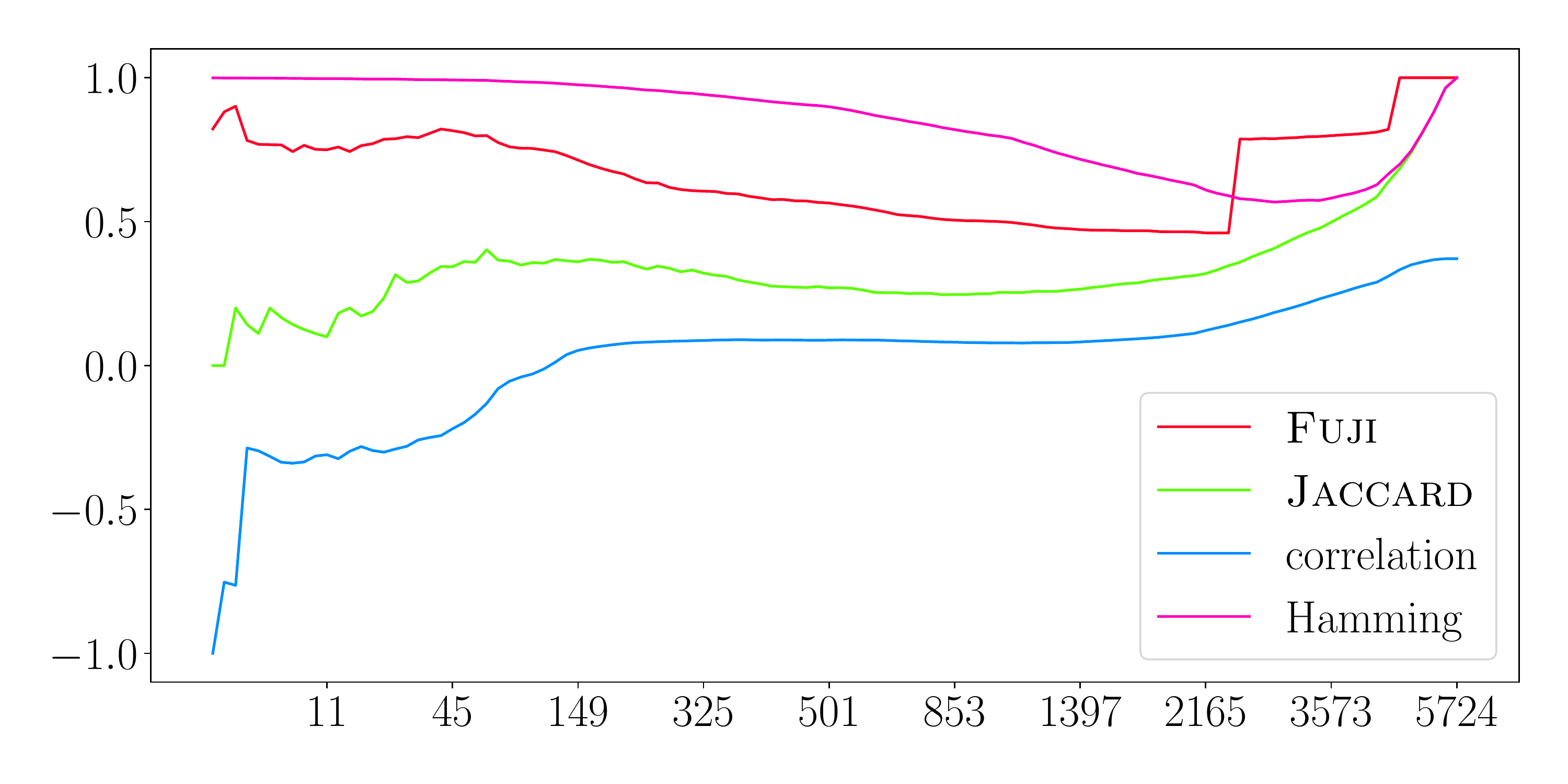}
\figureBlock{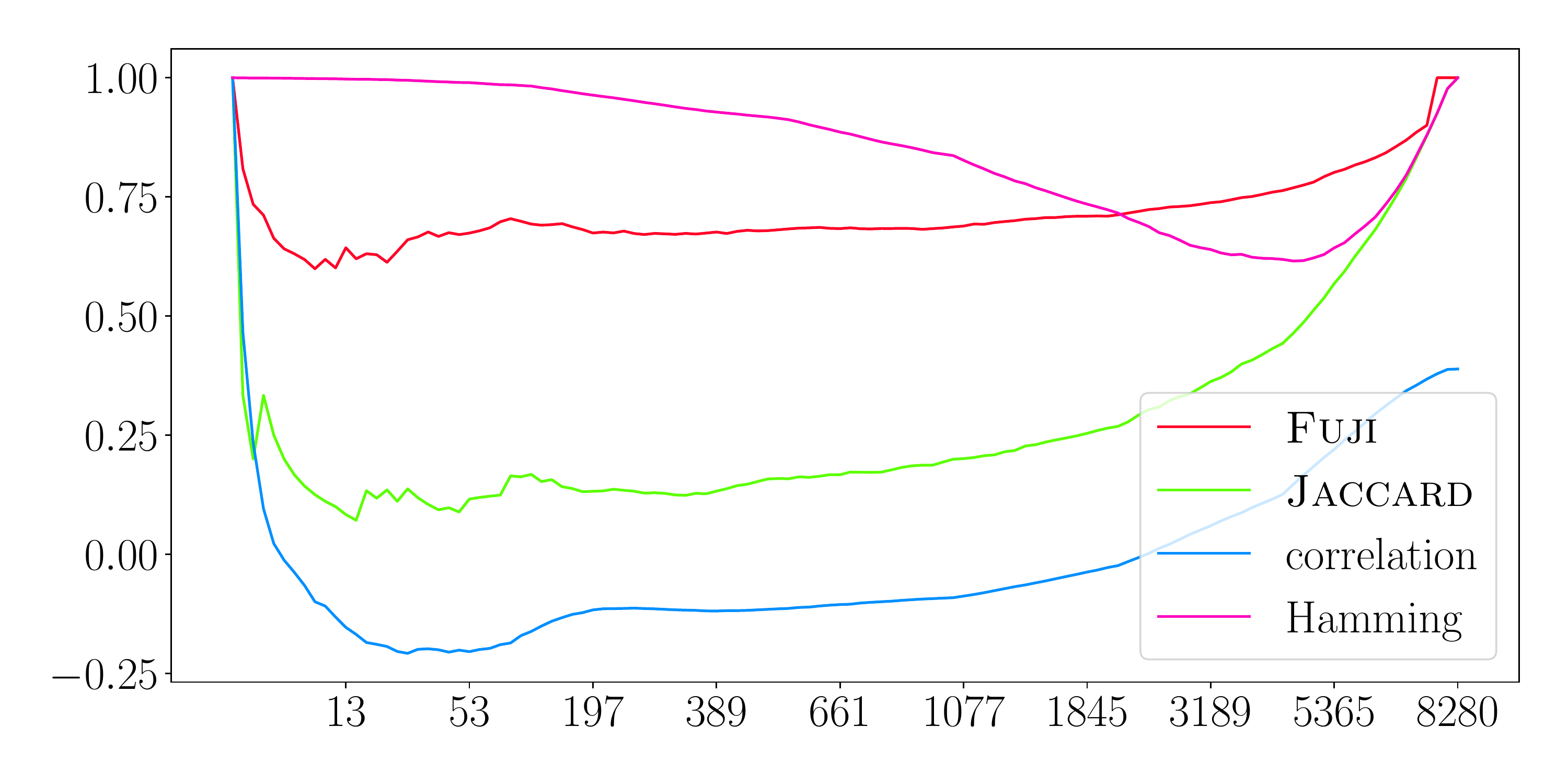}
\figureBlock{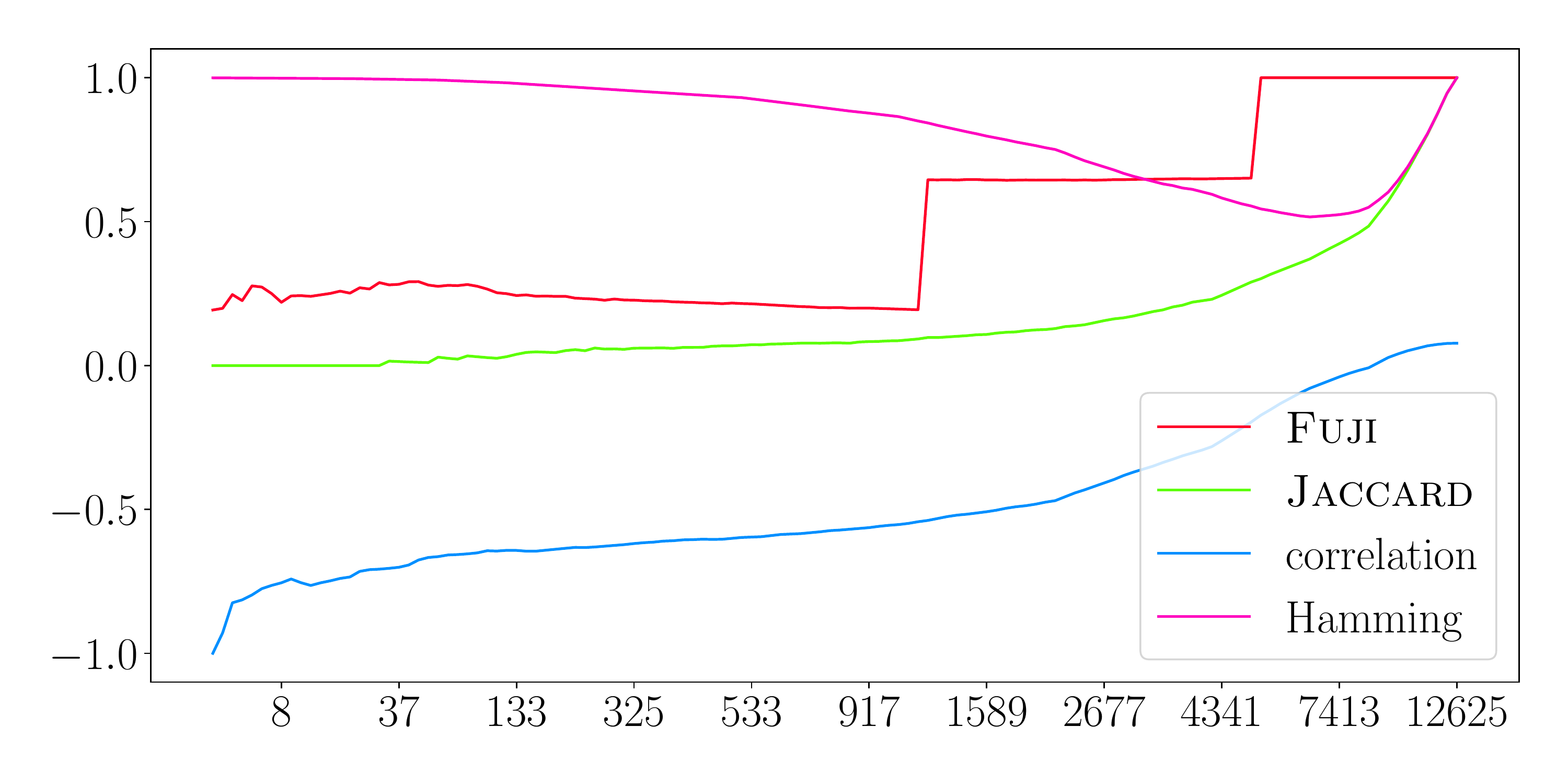}
\figureBlock{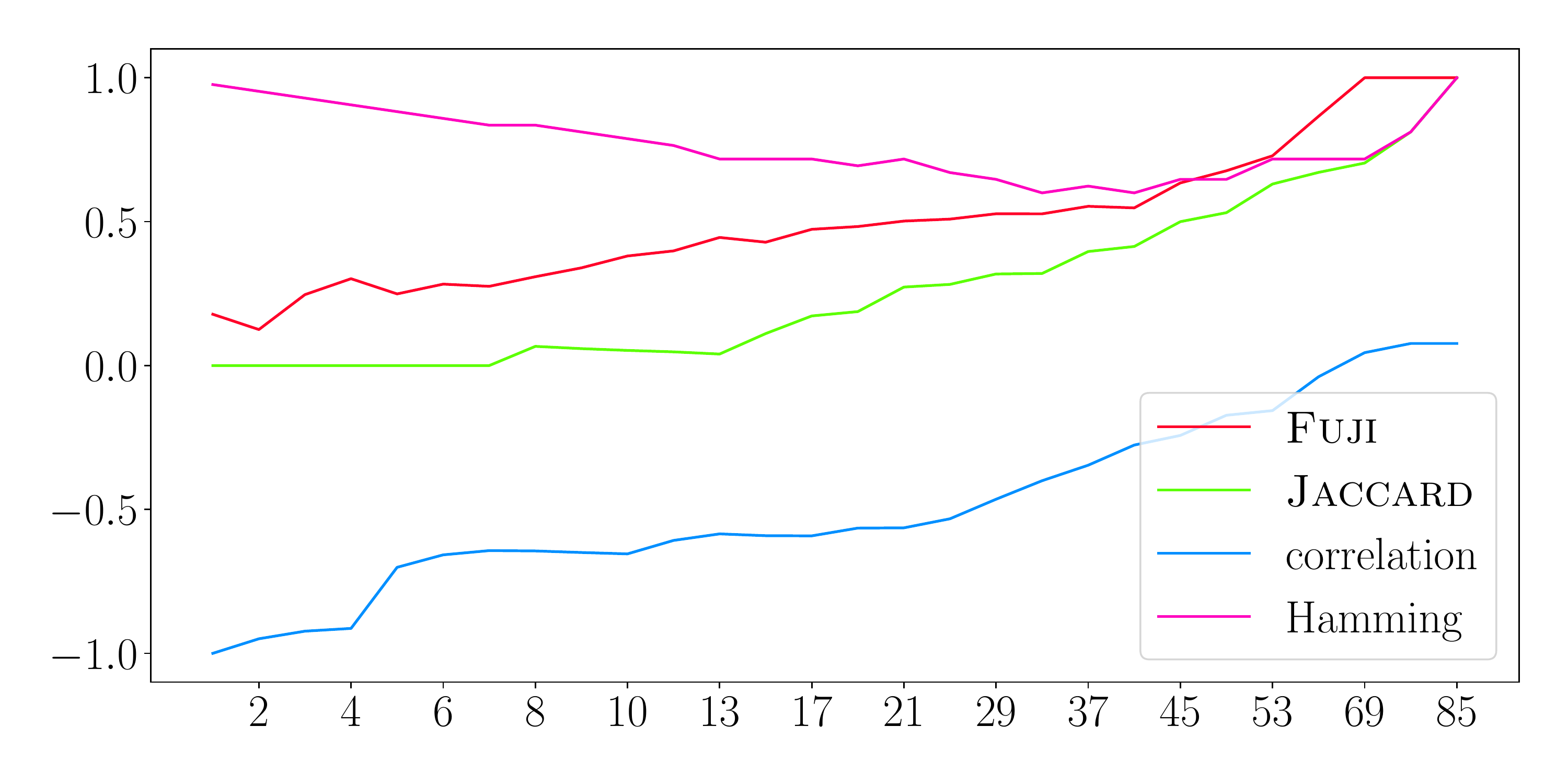}
\figureBlock{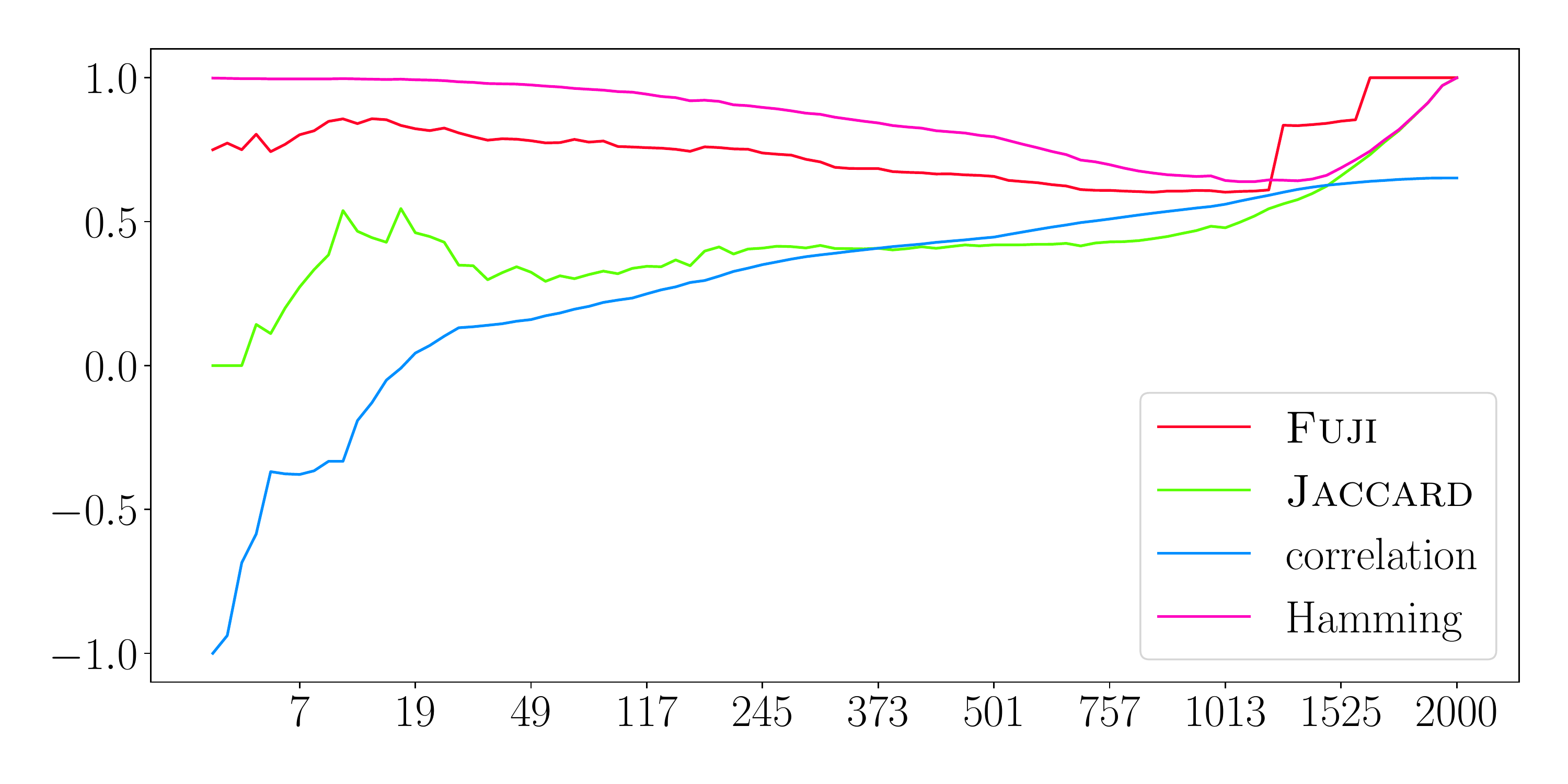}
\figureBlock{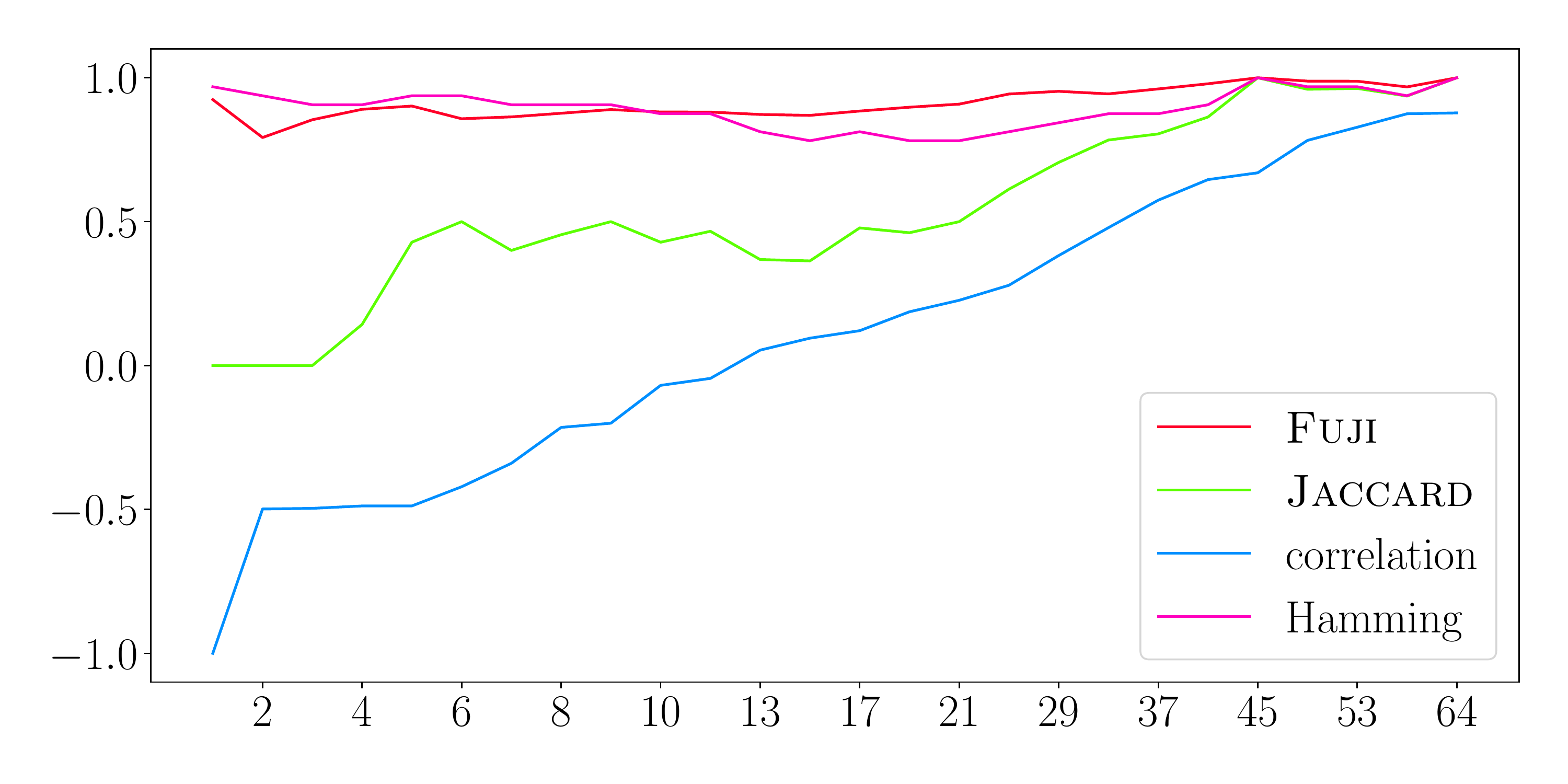}
\figureBlock{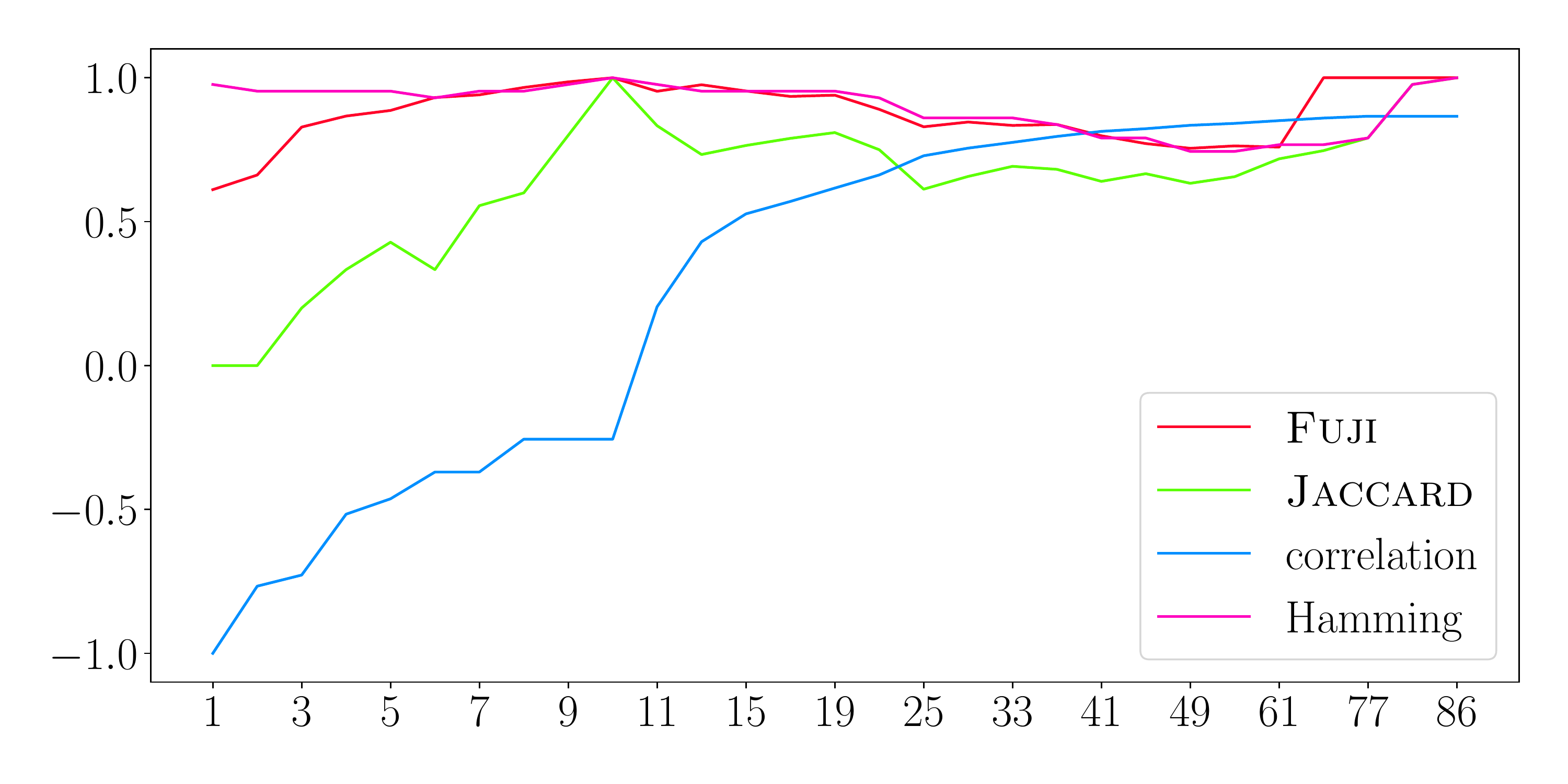}
\figureBlock{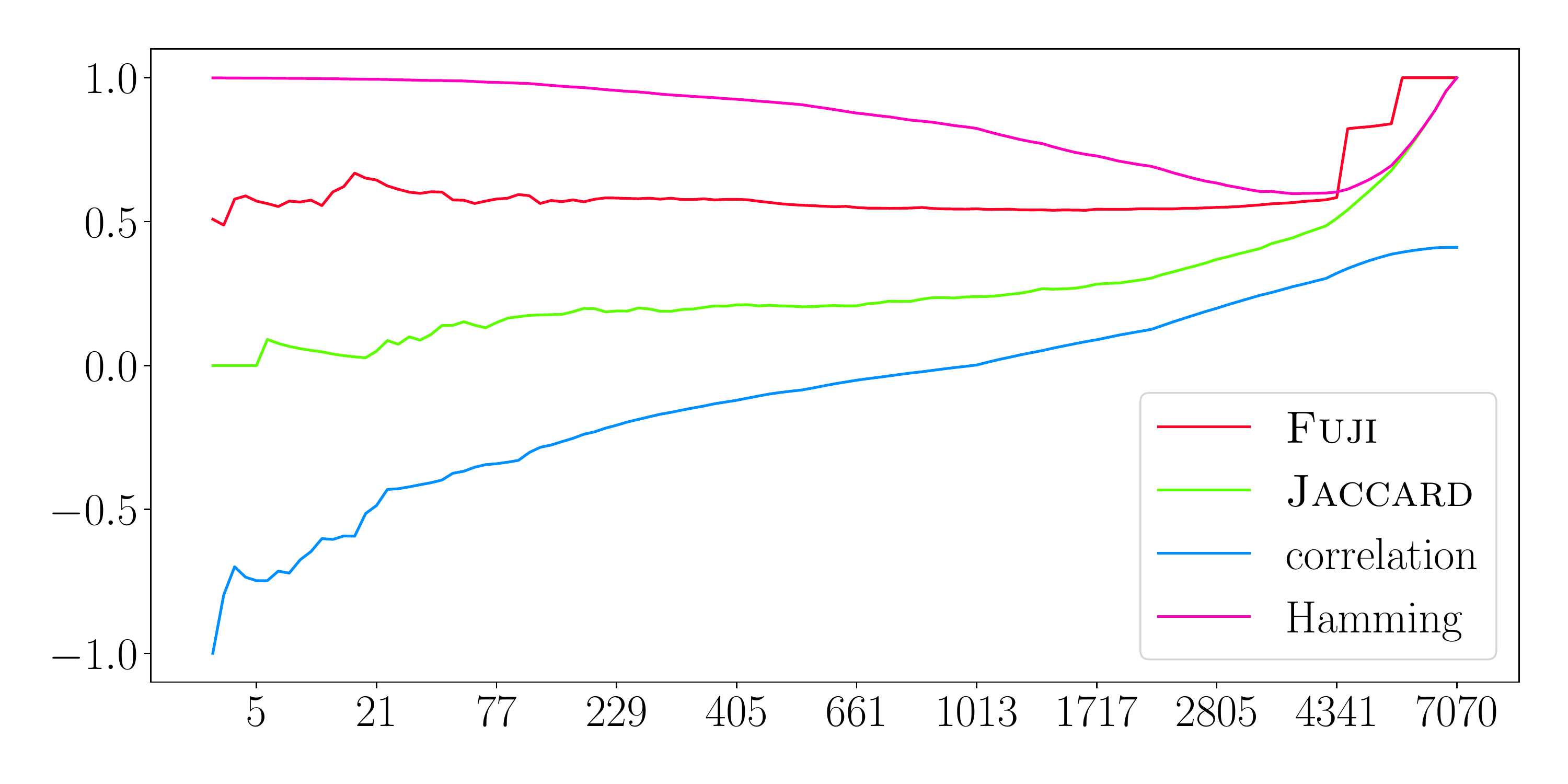}
\figureBlock{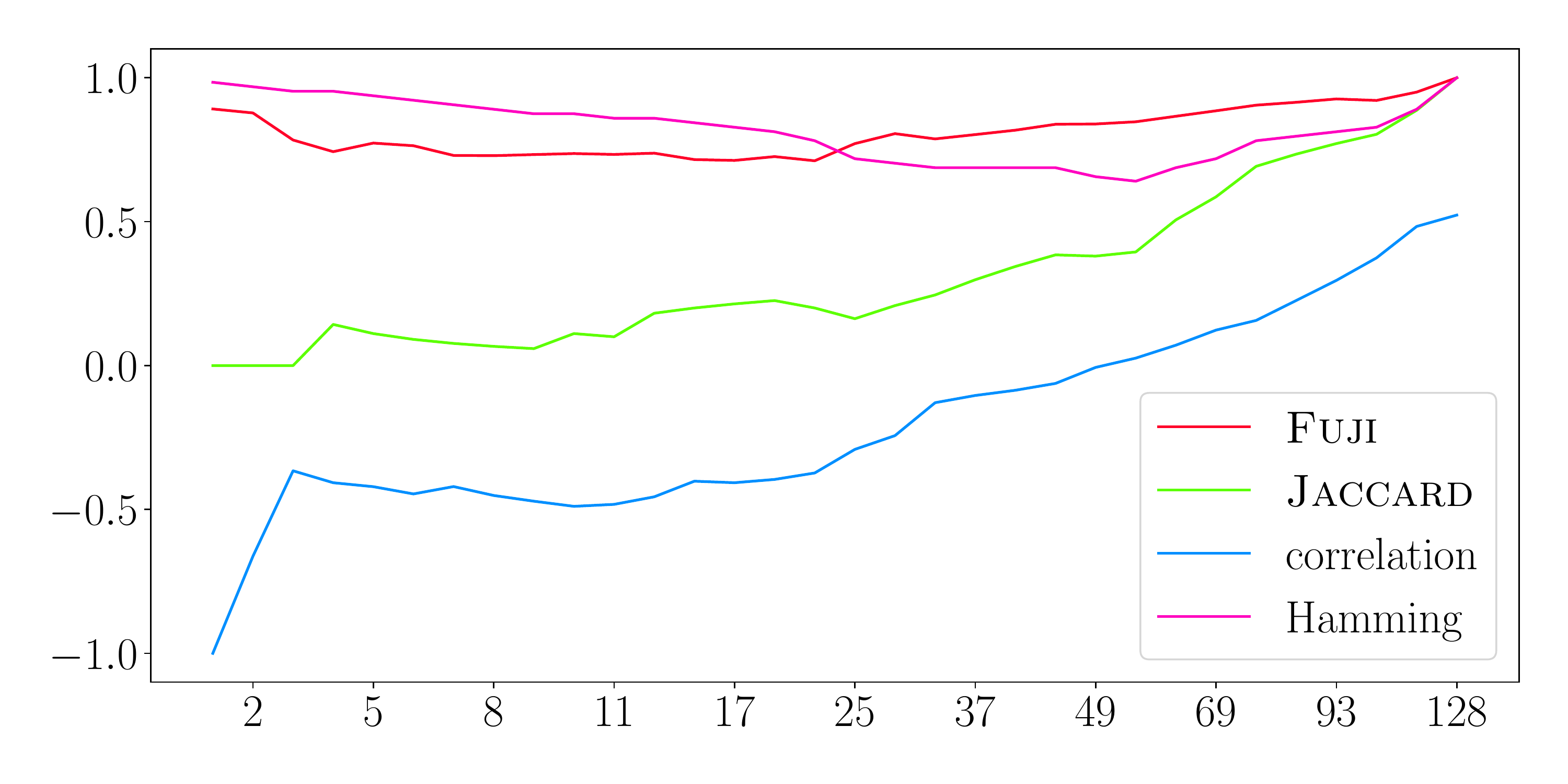}
\figureBlock{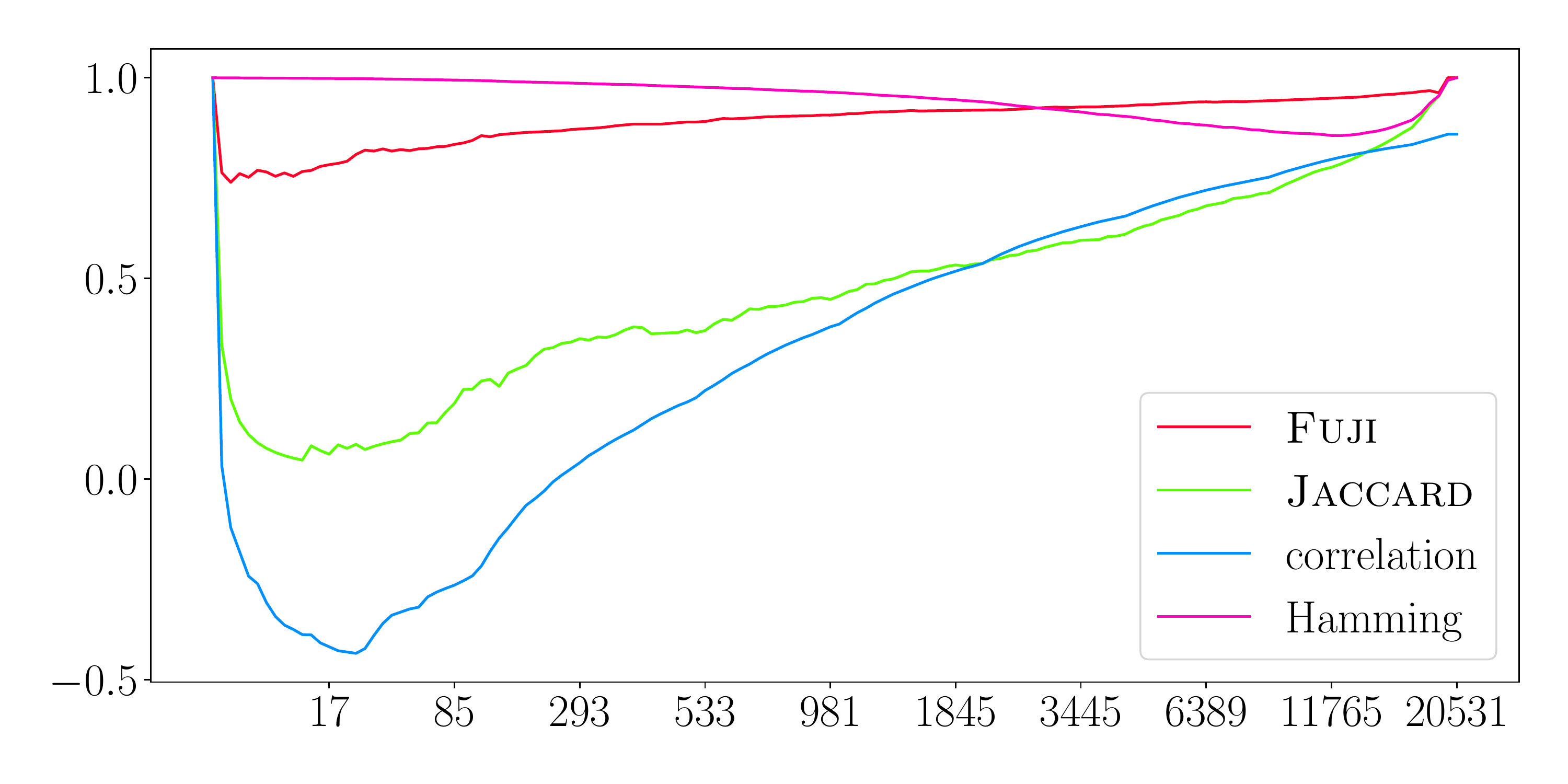}
\figureBlock{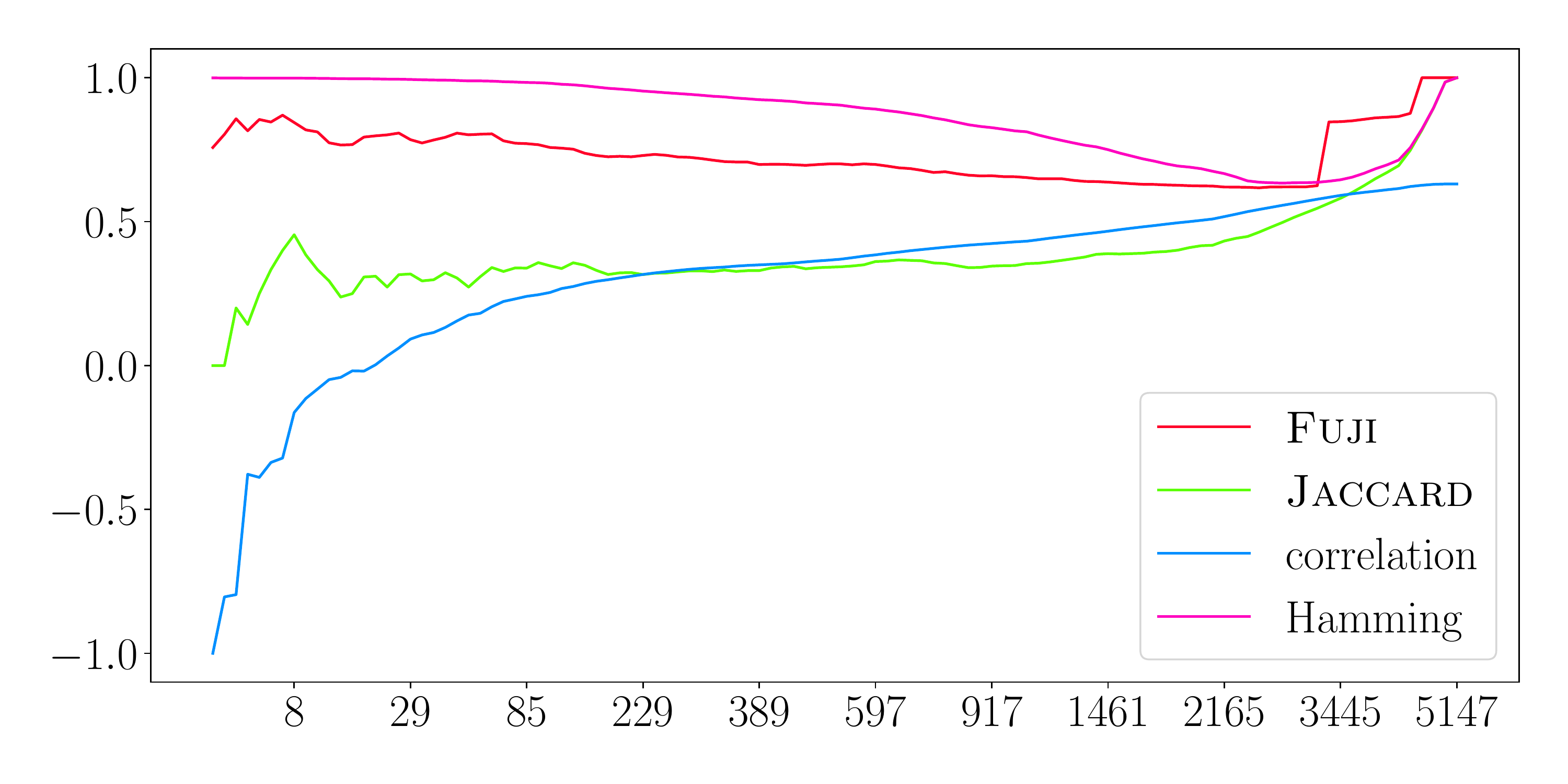}
\figureBlock{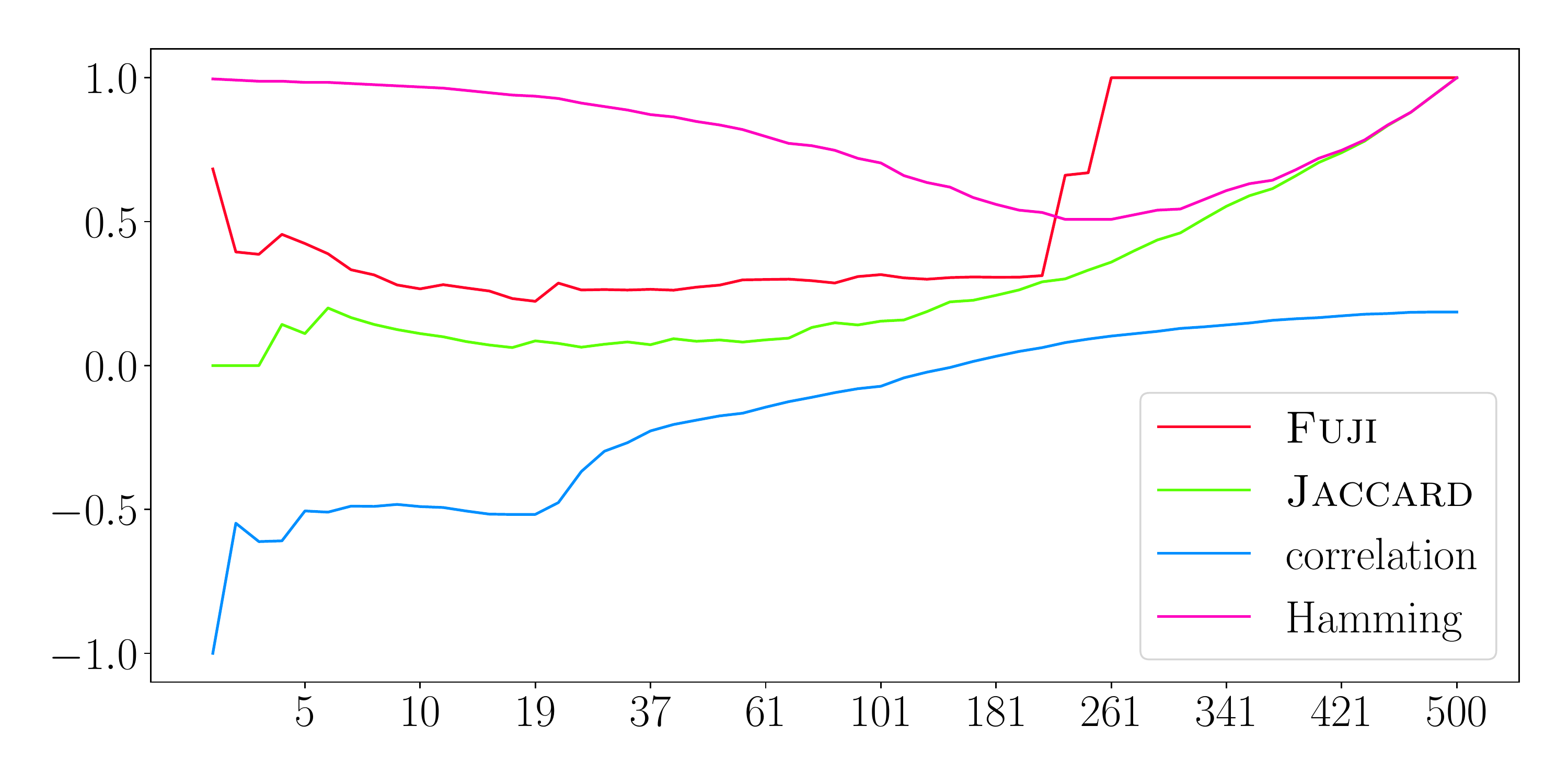}
\figureBlock{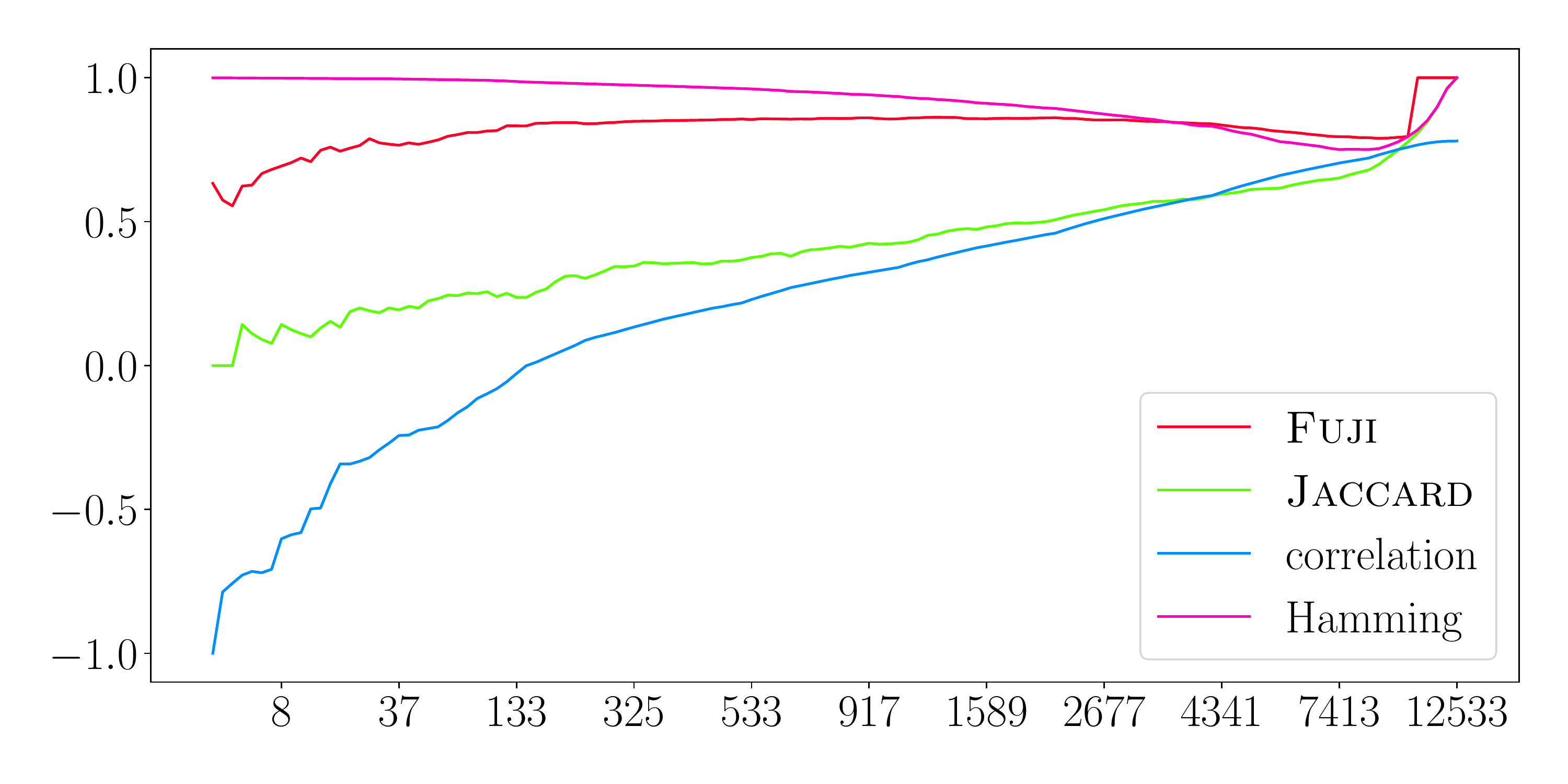}
\figureBlock{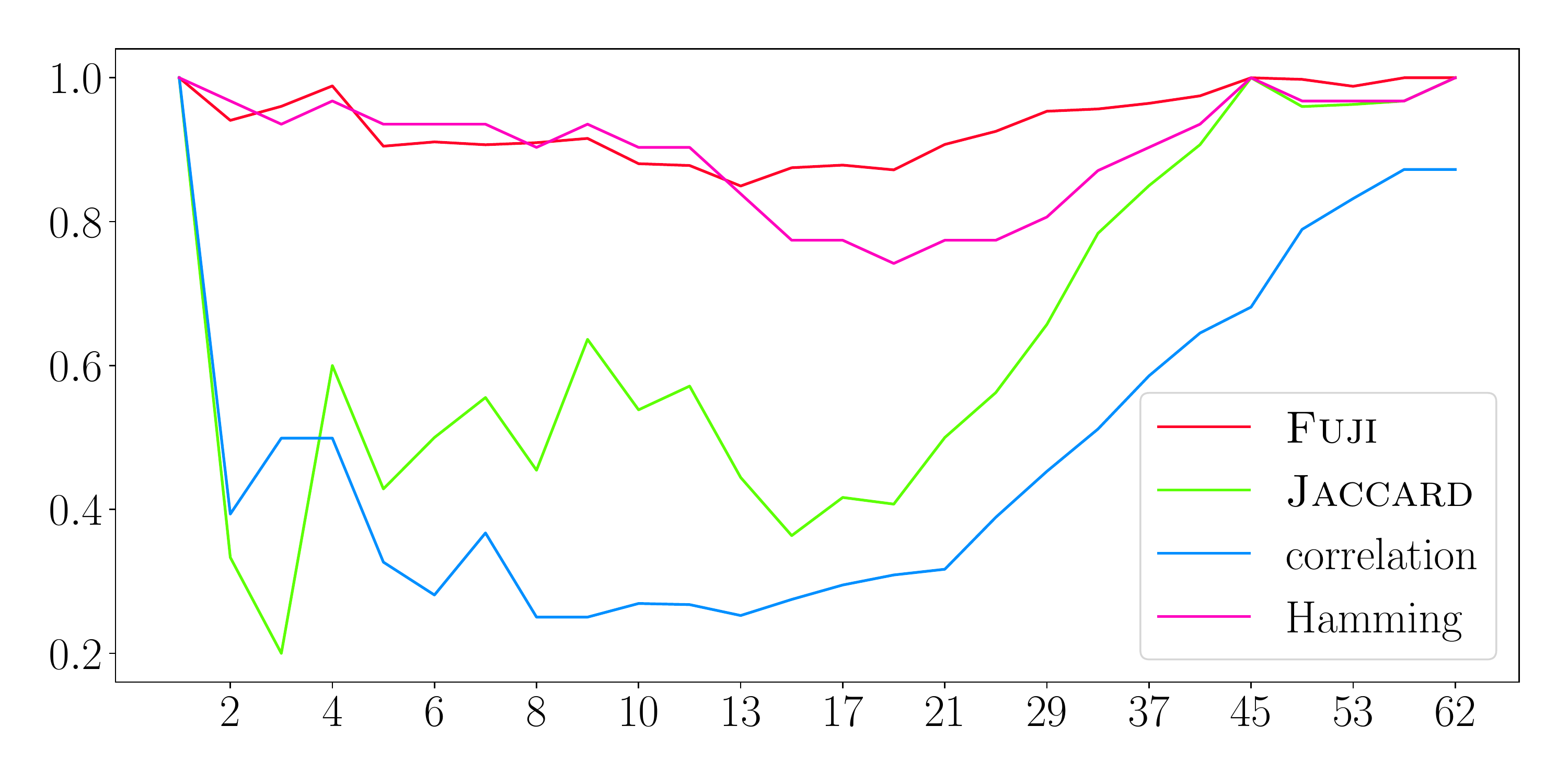}
\figureBlock{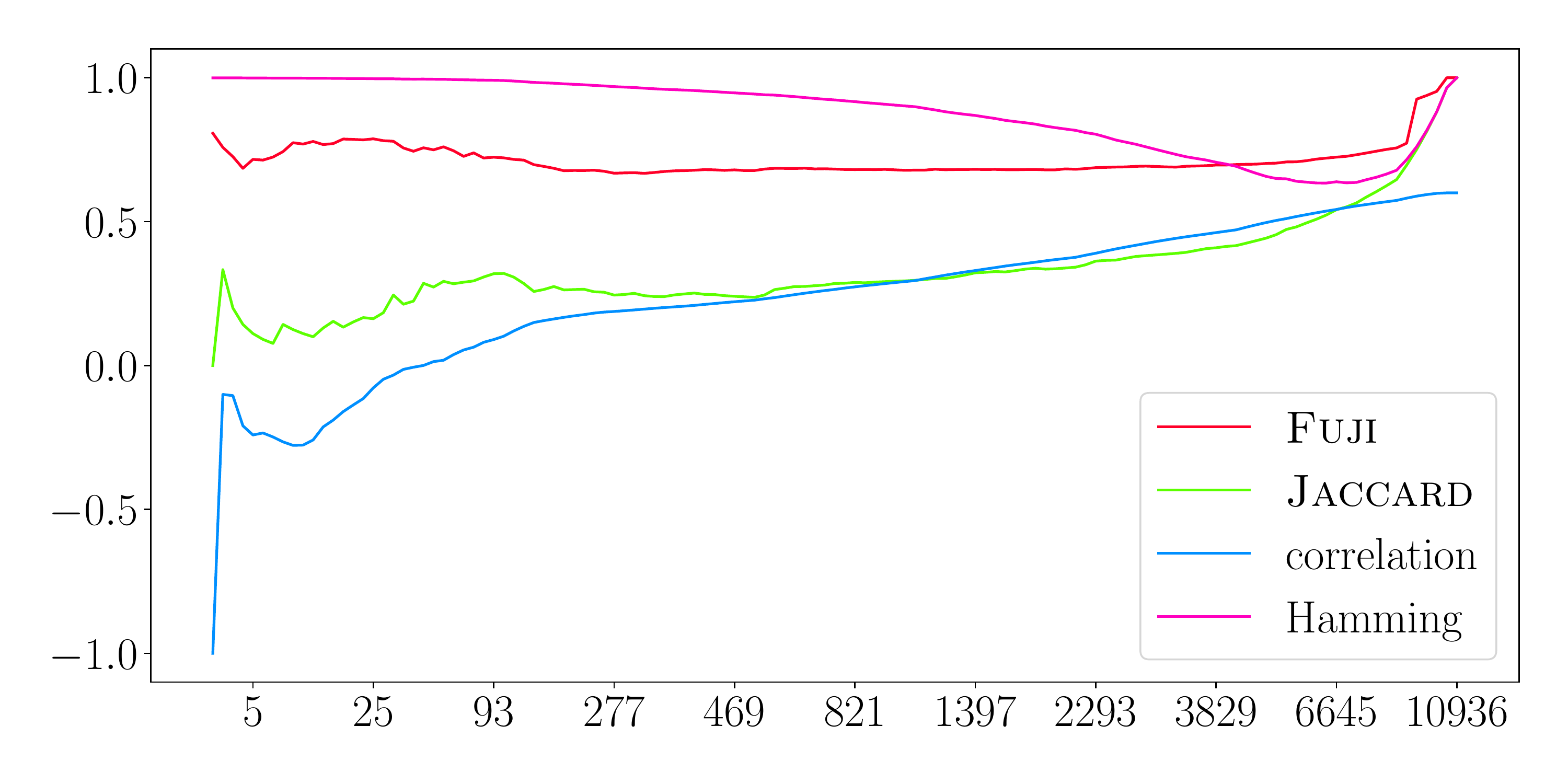}
\figureBlock{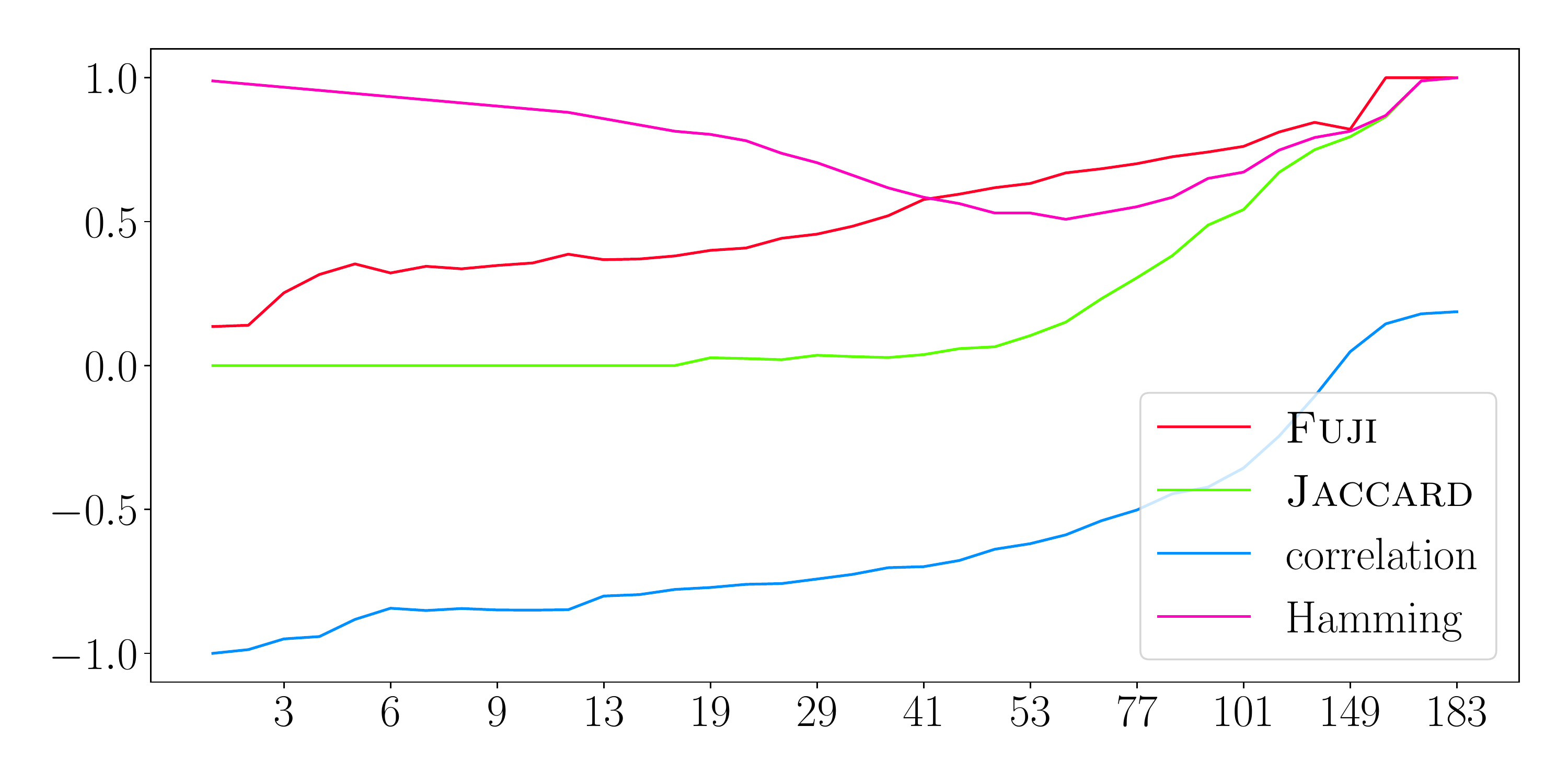}
\figureBlock{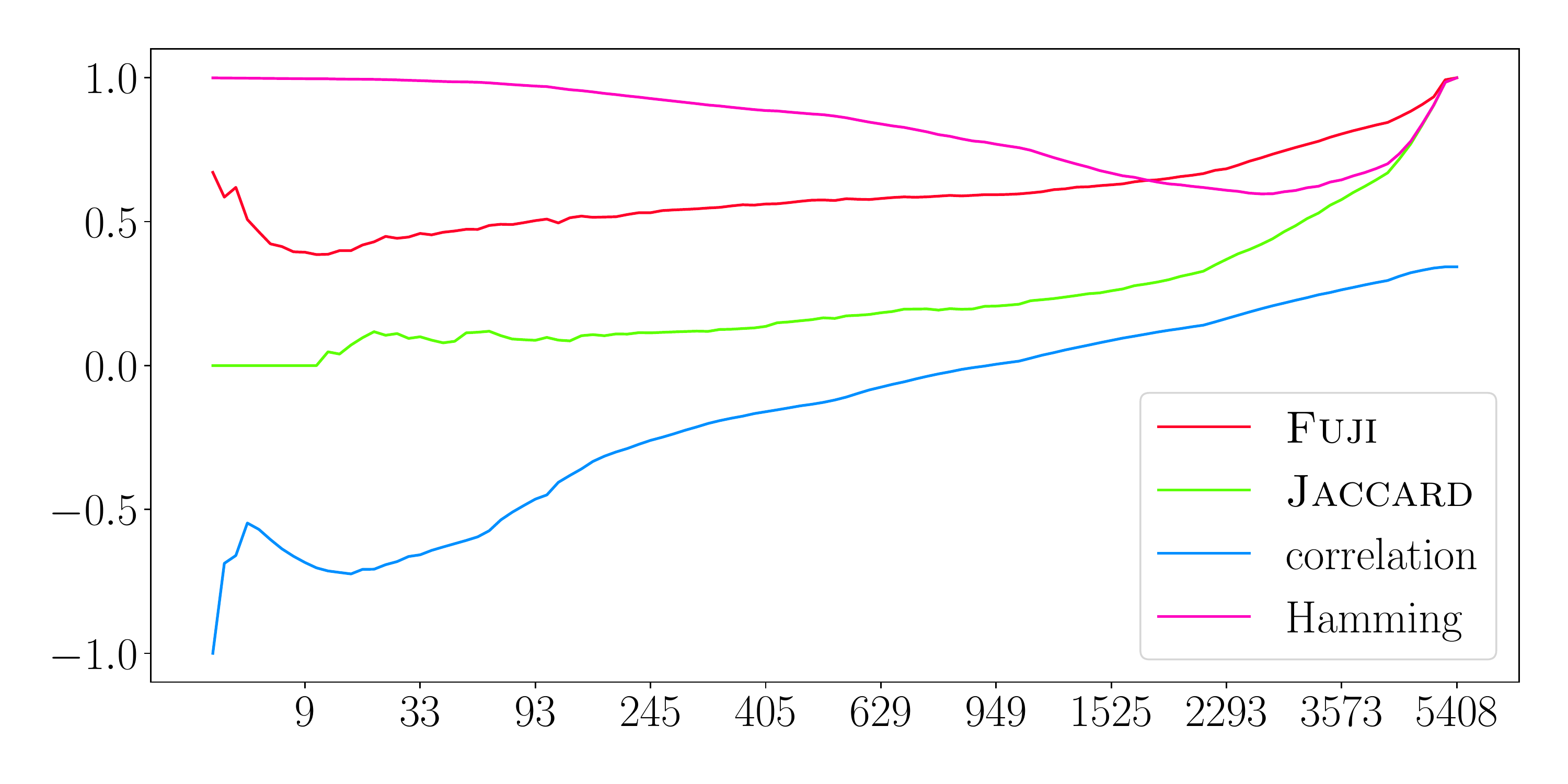}
\figureBlock{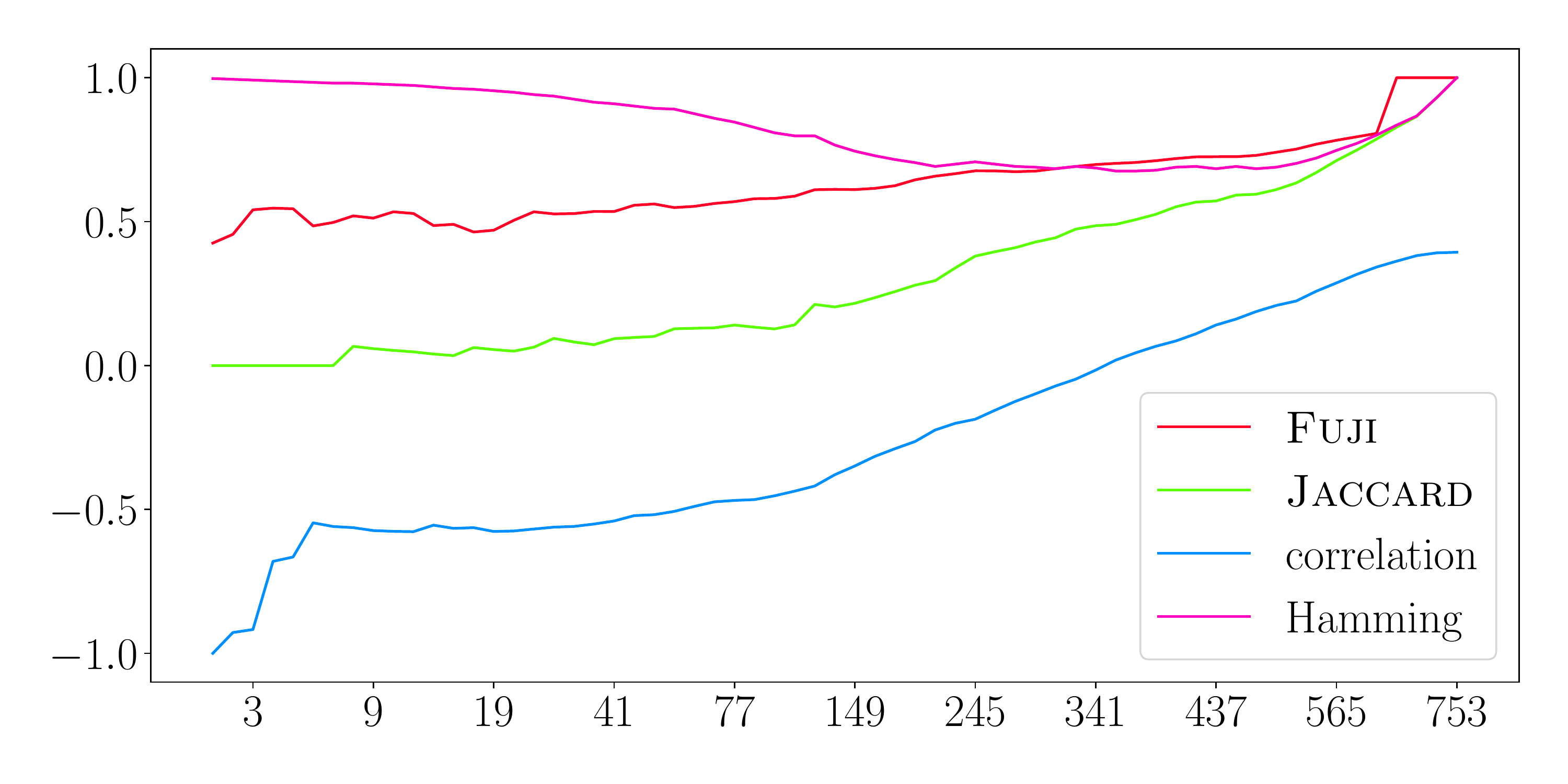}
\figureBlock{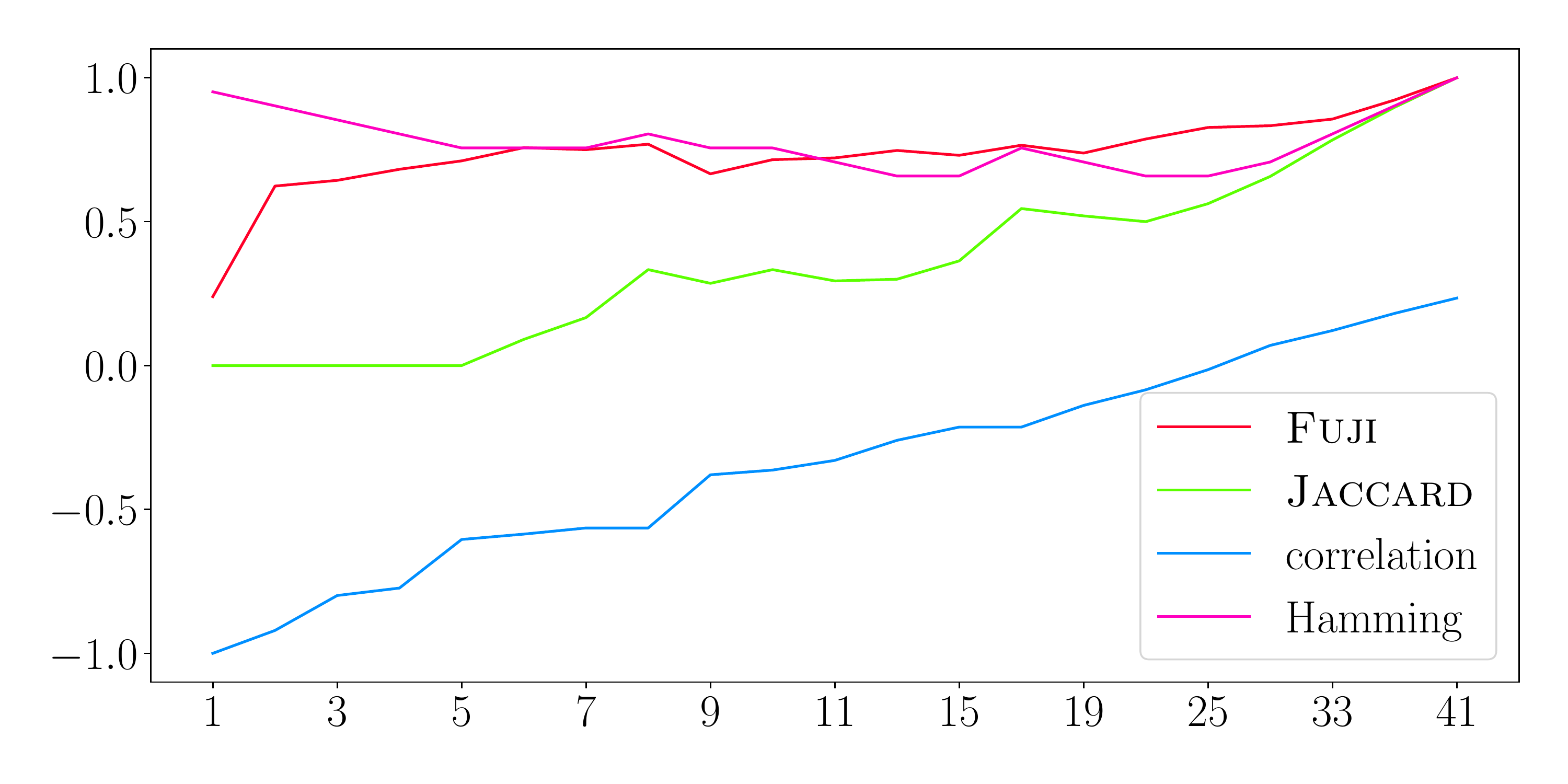}
\figureBlock{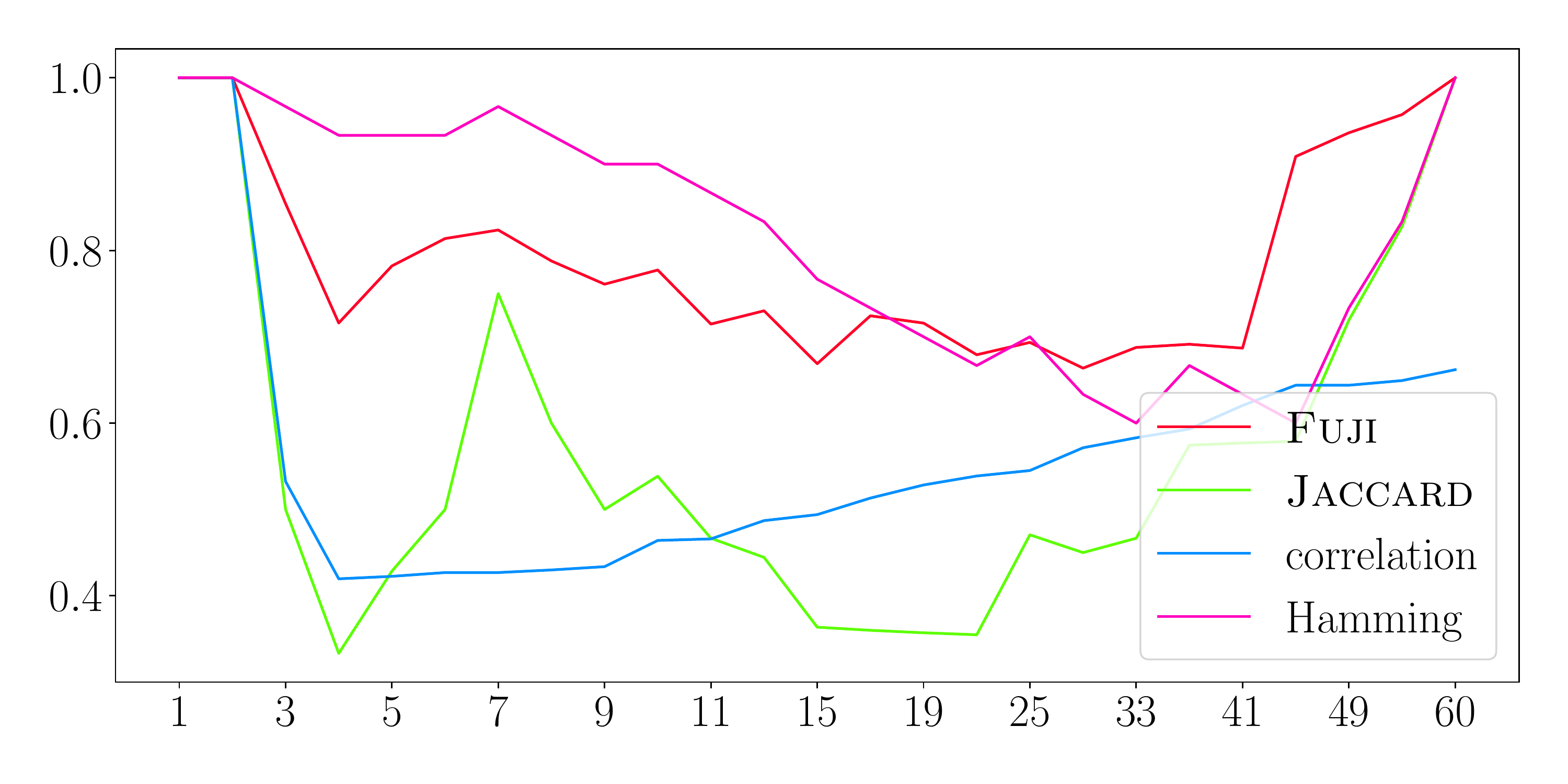}
\figureBlock{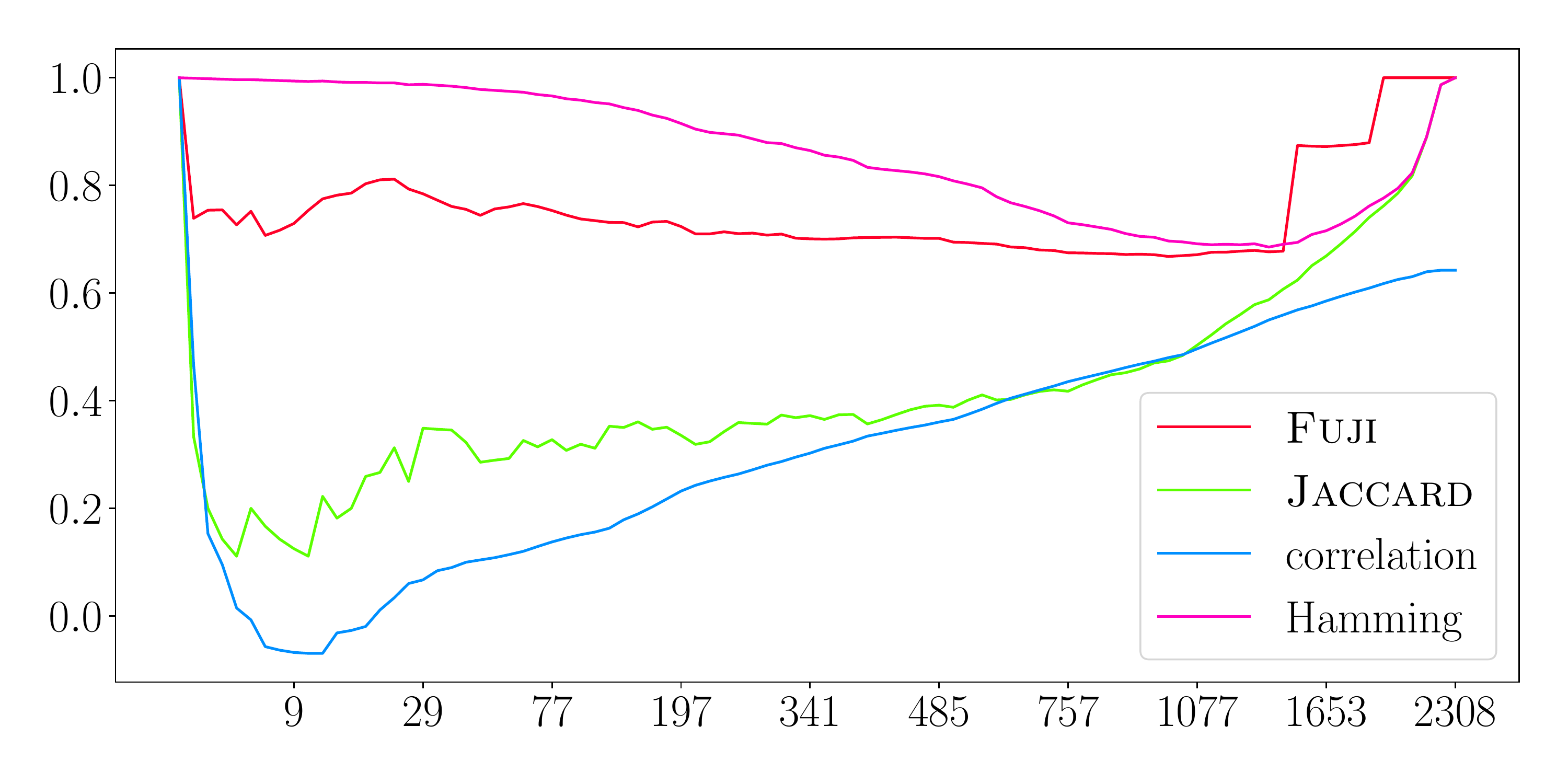}
\figureBlock{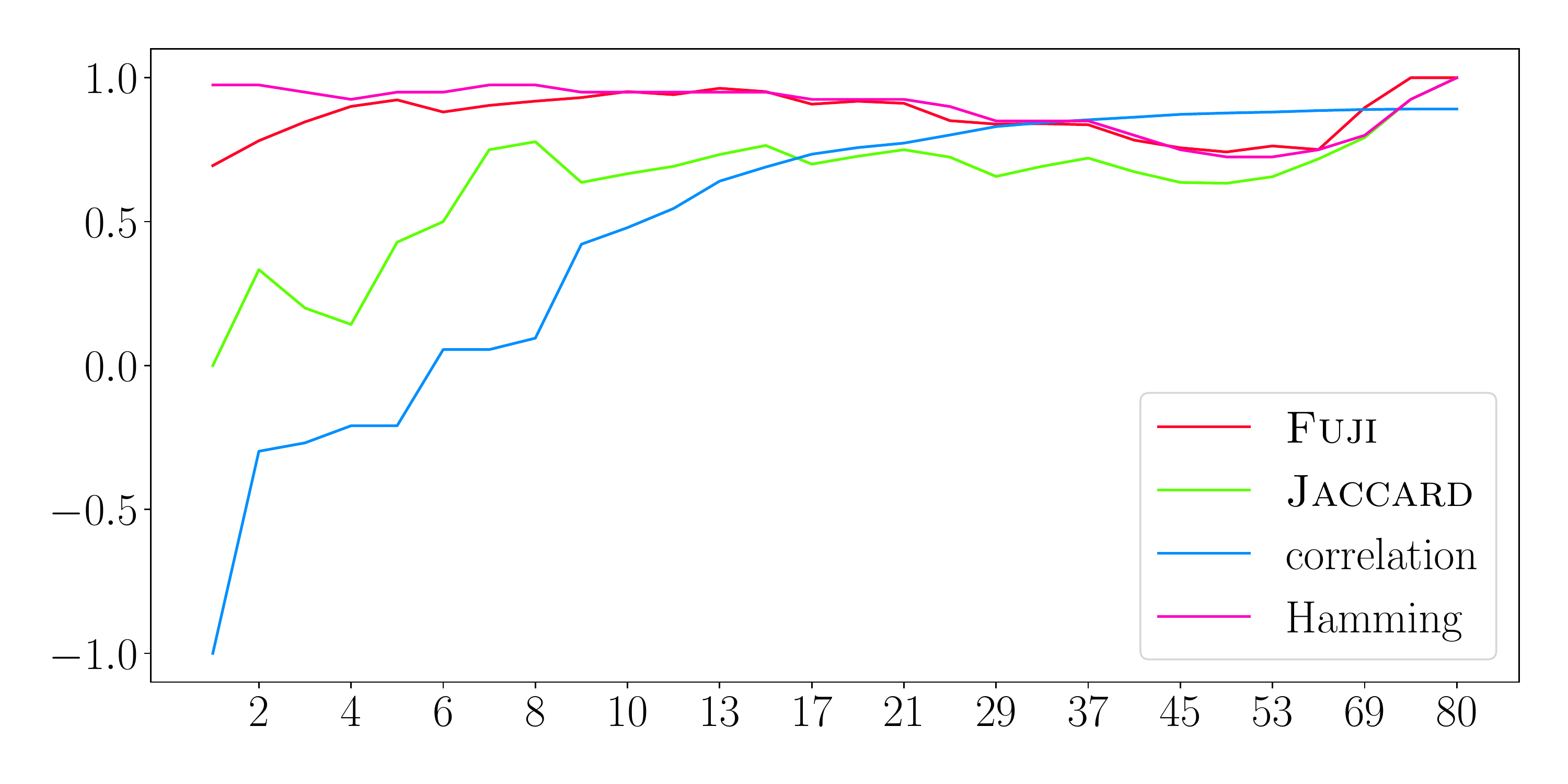}

\clearpage

\end{document}